\documentclass[twoside,11pt]{article}
\usepackage{jair, rawfonts}

\usepackage[natbibapa]{apacite}
\usepackage[usenames,dvipsnames]{xcolor}
\usepackage{tcolorbox}
\usepackage{tabularx}
\usepackage{array}
\usepackage{colortbl}
\tcbuselibrary{skins}
\usepackage{caption}
\usepackage{subcaption}
\usepackage{wrapfig}
\usepackage{lipsum}  
\usepackage{sidecap}
\usepackage{algorithm}
\usepackage[noend]{algpseudocode}
\usepackage{tikz}

\usepackage[hidelinks]{hyperref}

\usepackage{graphicx}
\usetikzlibrary{positioning}
\makeatletter
\def\BState{\State\hskip-\ALG@thistlm}
\makeatother

\usepackage{amsmath, amssymb}
\usepackage{tabularx}
\usepackage{booktabs}
\usepackage{xcolor}

\usepackage{longtable}
\usepackage{caption}
\usepackage{makecell}

\newcommand{\target}{c}
\newcommand{\targetr}{C}
\newcommand{\that}{\hat{c}}

\jairheading{1}{2019}{}{6/19}{}
\ShortHeadings{Adapting Behavior via Intrinsic Reward: a Survey and Empirical Study}
{Linke, Ady, White, Degris \& White}
\firstpageno{1}

\newcommand{\BlackBox}{\rule{1.5ex}{1.5ex}}  
\newenvironment{proof}{\par\noindent{\bf Proof\ }}{\hfill\BlackBox\\[2mm]}
 
\newtheorem{theorem}{Theorem}

\newcolumntype{Y}{>{\raggedleft\arraybackslash}X}

\tcbset{tab1/.style={fonttitle=\bfseries\large,fontupper=\normalsize\sffamily,
colback=yellow!10!white,colframe=red!75!black,colbacktitle=blue!0!white,
coltitle=black,center title,freelance,frame code={
\foreach \n in {north east,north west,south east,south west}
{\path [fill=red!75!black] (interior.\n) circle (3mm); };},}}

\tcbset{tab2/.style={enhanced,fonttitle=\bfseries,fontupper=\normalsize\sffamily,
colback=blue!2!white,colframe=blue!50!black,colbacktitle=blue!20,
coltitle=black,center title}}

\newcommand{\thetai}{{\theta(t,i)}}

\newcommand{\mui}{\mu_{t,i}}
\newcommand{\defeq}{\mathrel{\overset{\makebox[0pt]{\mbox{\normalfont\tiny\sffamily def}}}{=}}}

\begin{document}

\title{Adapting Behavior via Intrinsic Reward: \\A Survey and Empirical Study}

\author{\name Cam Linke \email clinke@ualberta.ca \\
	\name Nadia M. Ady \email nmady@ualberta.ca \\
	\name Martha White \email whitem@ualberta.ca\\
       \addr 
       Alberta Machine Intelligence Institute (Amii)\\
       Department of Computing Science\\
       University of Alberta \\
	\name Thomas Degris \email degris@google.com \\ 
       \addr DeepMind  \\
	\name Adam White \email adamwhite@google.com \\
       \addr 
       Alberta Machine Intelligence Institute (Amii)\\
       Department of Computing Science\\
        University of Alberta\\
        DeepMind Alberta\\
}


\maketitle

\begin{abstract}

Learning about many things can provide numerous benefits to a reinforcement learning system.  For example, learning many auxiliary value functions, in addition to optimizing the environmental reward, appears to improve both exploration and representation learning. The question we tackle in this paper is how to sculpt the stream of experience---how to adapt the learning system's behavior---to optimize the learning of a collection of value functions. A simple answer is to compute an intrinsic reward based on the statistics of each auxiliary learner, and use reinforcement learning to maximize that intrinsic reward. Unfortunately, implementing this simple idea has proven difficult, and thus has been the focus of decades of study. It remains unclear which of the many possible measures of learning would work well in a parallel learning setting where environmental reward is extremely sparse or absent. In this paper, we investigate and compare different intrinsic reward mechanisms in a new bandit-like parallel-learning testbed. We discuss the interaction between reward and prediction learners and highlight the importance of introspective prediction learners: those that increase their rate of learning when progress is possible, and decrease when it is not. We provide a comprehensive empirical comparison of 14 different rewards, including well-known ideas from reinforcement learning and active learning. Our results highlight a simple but seemingly powerful principle: intrinsic rewards based on the amount of learning can generate useful behavior, if each individual learner is introspective.
\end{abstract}

\section{Balancing the Needs of Many Learners}

Learning about many things can provide numerous benefits to a reinforcement learning system. Adding many auxiliary losses to a deep learning system can act as a regularizer on the representation, ultimately resulting in better final performance in reward maximization problems, as demonstrated with Unreal \citep{jaderberg2016reinforcement}. A collection of value functions encoding goal-directed behaviors can be combined to generate new policies that generalize to goals unseen during training \citep{schaul2015universal}. Learning in hierarchical robot-control problems can be improved with persistent exploration, given call-return execution of a collection of subgoal policies or skills \citep{riedmiller2018learning}, even if those policies are imperfectly learned. In all these examples, a collection of general value functions \citep[see][]{sutton2011horde} is updated from a single stream of experience. The question we tackle in this paper is how to sculpt that stream of experience---how to adapt the learning system's behavior---to optimize the learning of a collection of value functions. 

One answer is to simply maximize the environmental (extrinsic) reward. This was the approach explored in Unreal and it resulted in significant performance improvements in challenging visual navigation problems. However, it is not hard to imagine situations where this approach would be limited. In general, the reward may be delayed and sparse: what should the agent do in the absence of external (environmental) motivations? Learning reusable knowledge such as skills \citep{sutton1999between} or a model of the world might result in more long-term reward. Such auxiliary learning objectives could emerge automatically during learning \citep{silver2017predictron}. Most agent architectures, however, include explicit skill and model learning components.  It seems natural that progress towards these auxiliary learning objectives could positively influence the agent's behavior, resulting in improved learning overall. 

Learning many value functions off-policy from a shared stream of experience---with function approximation and an unknown environment---provides a natural setting to investigate no-reward intrinsically motivated learning. The basic idea is simple. The aim is to accurately estimate many value functions, each with an independent learner---where there is no external reward signal. Directly optimizing the data collection for all learners jointly is difficult because we cannot directly measure this total learning objective and because actions have an indirect impact on learning efficiency. There is a large related literature in active learning \citep{cohn1996active,balcan2009agnostic,settles2009active,golovin2011adaptive,konyushkova2017learning} and active perception \citep{bajcsy2018revisiting},
from which to draw inspiration for a solution but which do not directly apply to this problem.
In active learning the agent must sub-select from a larger set of items to choose which points to label. Active perception is a subfield of vision and robotics. Much of the work in active perception has focused on specific settings---namely visual attention \citep{bylinskii2015towards}, localization in robotics \citep{patten2018monte} and sensor selection \citep{satsangi2018exploiting,satsangi2020maximizing}---or assumes knowledge of the dynamics of the world \citep[see][]{bajcsy2018revisiting}. 

An alternative strategy is to formulate our task as a reinforcement learning problem. We can use an intrinsic reward, internal to the learning system, that approximates the total learning across all learners. 
The behavior can be adapted to choose actions in each state that maximize the intrinsic reward, towards the goal of maximizing the total learning of the system. The choice of intrinsic rewards can have a significant impact on the sample efficiency of such intrinsically motivated learning systems. This paper provides the first formulation of parallel value function learning as a reinforcement learning task. Fortunately, there are many ideas from related areas that can inform our choice of intrinsic rewards.

Rewards computed from internal statistics about the learning process have been explored in many contexts over the years.
Intrinsic rewards have been shown to induce behavior that resembles the development stages exhibited by young humans and animals \citep{barto2013intrinsic,chentanez2005intrinsically,oudeyer2007intrinsic,lopes2012exploration,haber2018learning}.
Internal measures of learning have been used to improve skill or option learning \citep{chentanez2005intrinsically,schembri2007evolving,barto2005intrinsic,santucci2013intrinsic,vigorito2016intrinsically}, and model learning \citep{schmidhuber1991possibility,schmidhuber2008driven}.
Most recent work has investigated using intrinsic reward as a bonus to encourage additional exploration in single task learning \citep{itti2006bayesian,stadie2015incentivizing,bellemare2016unifying,pathak2017curiosity,hester2017intrinsically,tang2017exploration,andrychowicz2017hindsight,achiam2017surprise,martin2017count,colas2018gep,schossau2016information,pathak2019self}. Few have investigated the impact of making these internal measures the main objective of learning \citep{berseth2019smirl}, however previous studies have noted that intrinsic reward is useful even in single-task problems with a well-defined external goal~\citep{bellemare2016unifying}.  

It remains unclear, however, which of these measures of learning would work best in our no-reward setting. Most prior work has focused on providing demonstrations of the utility of particular intrinsic reward mechanisms. One study focused on a suite of large-scale control domains with a single scalar external reward~\citep{pathak18largescale}, comparing different learning systems that use an intrinsic reward based on model-error as an exploration bonus. 
A large study has been conducted on learning progress measures for curriculum learning for neural networks \citep{graves2017automated}, where the goal is to learn from which task to sample a dataset to update the parameters. Variants of their measures are related to the intrinsic rewards explored in this paper, but their setting differs substantially in that learning is offline from batch supervised learning datasets and the underlying problems are stationary.
To the best of our knowledge, there has never been a broad empirical comparison of intrinsic rewards for the online multi-prediction setting with non-stationary targets.

A computational study of intrinsic rewards is certainly needed, but tackling this problem with function approximation and off-policy updating is not the place to start. Estimating multiple value functions in parallel requires off-policy algorithms because each value function is conditioned on a policy that is different than the exploratory behavior used to select actions. In problems of moderate complexity, these off-policy updates can introduce significant technical challenges. Popular off-policy algorithms like Q-learning and V-trace can diverge with function approximation \citep{sutton2018reinforcement}. Sound off-policy algorithms exist, but require tuning additional parameters and are relatively understudied in practice. Even in tabular problems, good performance requires tuning the parameters of each component of the learning system---a complication that escalates with the number of value functions. Finally, the agent must solve the primary exploration problem in order to make use of intrinsic rewards. Finding states with high intrinsic reward may not be easy, even if we assume the intrinsic reward is reliable and informative. To avoid these many confounding factors, the right place to start is in a simpler setting.

In this paper, we investigate and compare different intrinsic reward mechanisms in a new bandit-like parallel learning testbed. The testbed consists of a single state and multiple actions. Each action is associated with an independent scalar target to be estimated by an independent prediction learner. A reasonable behavior policy will focus on actions that generate the most learning across the prediction learners. However, the overall task is partially observable, and learning is never done. The targets change without an explicit notification to the agent, and the task continually changes due to changes in action selection and learning of the individual prediction learners.  Different configurations of the target distributions can simulate unlearnable targets, non-stationary targets, and easy-to-predict targets. Our new testbed provides a simple instantiation of a problem where \emph{introspective} learners should help achieve low overall error.  An introspective prediction learner is one that can autonomously increase its rate of learning when progress is possible, and decrease learning when progress is not---or cannot---be made. 

This paper summarizes a comprehensive empirical comparison of different intrinsic reward mechanisms, including several ideas from reinforcement learning and active learning. 
%
Our computational study of learning progress highlighted a simple principle: intrinsic rewards based on the {\em amount of learning} (e.g., Bayesian Surprise and simple change in weights) can generate useful behavior if each individual learner is introspective. Across a variety of problem settings we find that the combination of introspective learners and simple intrinsic rewards was most reliable, performant, and easy to tune.
We conclude with a discussion about how these ideas could be extended beyond our one-state prediction problem to drive behavior in large-scale problems where off-policy learning and function approximation are required.

\section{Problem Formulation}
\label{prob}
In this section we formalize a testbed for comparing intrinsic reward using a stateless prediction task and independent learners. This formalism is meant to simplify the study of balancing the needs of many learners to facilitate comprehensive comparisons. 

We formalize our multiple-prediction learning setting as a collection of independent, online supervised learning tasks.
On each discrete time step $t=1,2,3,...$, the {\em behavior agent} selects an action $A_t \in \{1,\ldots, N\}$ corresponding to task $i \in \{1,\ldots, N\}$\footnote{To clearly separate the action selected by the agent and the prediction task, we use $A_t$ to denote the action selected at time $t$ and $i$ to denote prediction task. $A_t$ is uppercase to indicate it is a random variable, with lowercase $a$ corresponding to an actual action selected. In our setting, taking action $a$ corresponds to observing data for task $i$, and so there is an equivalence between the actions and tasks. More generally, such as in the full reinforcement learning setting, this is not the case; for extensions on this work, it is useful to clearly delineate between actions and prediction tasks.}, causing a target signal to be sampled from an (unknown) target distribution, $\target_{t,i} \sim \thetai$, where $\targetr_{t,i}$ denotes the random variable with distribution $\thetai$.
This distribution $\thetai$ is indexed by time to reflect that it can change on each time step; this enables a wide range of different target distribution to be considered, to model this non-stationary, multi-prediction learning setting. We provide the definition we use in this work later in this section, in Equation \eqref{drift}.

Associated with each prediction task is a simple {\em prediction learner} that maintains a real-valued vector of weights $w_{t,i}$, to produce an estimate, $\that_{t,i}\in\mathbb{R}$, of the expected value of the target, $\that_{t,i} \approx \mathbb{E}[\targetr_{t,i}]$. On a step where task $i$ is selected, $w_{t,i}$ could be updated using any standard learning algorithm. In this work, we use a 1-dimensional weight vector, and so the update is a simple delta-rule (least-mean-squares (LMS) learners):
\begin{equation}
w_{t+1,i} \leftarrow w_{t,i} + \alpha_{t,i}\delta_{t,i}  \label{lms}
\end{equation}
where $\alpha_{t,i}$ is a scalar step-size parameter and $\delta_{t,i} \defeq \target_{t,i} - w_{t,i}$ is the prediction error of prediction learner $i$ on step $t$. On a step where task $i$ is not selected, $w_{t,i}$ is not updated, implicitly setting $w_{t+1,i}$ to $w_{t,i}$.

The primary goal is to minimize the Mean Squared Error up to time $t$ for all of the $N$ learners:
\begin{equation} {\tt{MSE}}(t) \defeq \frac{1}{t} \sum_{k=1}^{t} \frac{1}{N} \sum_{i=1}^N  (\that_{k,i}- \mathbb{E}[\targetr_{k,i}])^2\label{eq:MSE_arms}
.
\end{equation}
The behavior agent does not get to observe this error, both because it only observes one of the targets $\target_{k,i}$ on each step, rather than all $N$, and because that target is a noisy sample of the true expected value $\mathbb{E}[\targetr_{k,i}]$. The agent can nonetheless attempt to minimize this unobserved error.

In order to minimize Equation \eqref{eq:MSE_arms}, we must devise a way to choose which prediction task to sample. This can be naturally formulated as a sequential decision-making problem, where on each time step $t$, the behavior agent chooses an task $i$ (corresponding to action $A_t$), resulting in a new sample of $\target_{t,i}$, and an update to $w_{t,i}$. In order to learn a preference over actions we associate an intrinsic reward $R_t \in \mathbb{R}$ with each action selection, and thus with each prediction task. We investigate different intrinsic rewards. Given a definition of the intrinsic reward, we can use a bandit algorithm suitable for non-stationary problems; we discuss two options below in Section \ref{sec_banditalgs}.

The targets for each prediction learner are intended to replicate the dynamics of targets that a parallel auxiliary task learning system might experience, such as sensor values of a robot. To simulate a range of interesting dynamics, we construct each $\thetai$ as a Gaussian distribution with drifting mean:
\begin{align}
&\thetai \defeq \mathcal{N}(\mui,\sigma^2_{t,i}) \label{drift}\\
&\text{for } \ \
\mu_{t+1,i} \leftarrow \Pi_{[-50,50]}\left(\mui + \mathcal{N}(0,\xi^2_{t,i}) \right) \nonumber
\end{align}
where  $\mui\in\mathbb{R}$, $\sigma^2_{t,i}\in\mathbb{R}^+$ controls the sampling noise, $\xi^2_{t,i}\in\mathbb{R}^+$ controls the rate of drift and $\Pi_{[-50,50]}$ projects the drifting $\mu_{t,i}$ back to the range $[-50,50]$ to keep it bounded. The variance and drift are indexed by $t$ because we explore settings where they change periodically. These changes are not communicated to the behavior agent, and the individual LMS learners are prevented from storing explicit histories of the targets. The purpose of this choice was to simulate partial observability common in many large-scale systems \citep[e.g.,][]{sutton2011horde,modayil2014multi,jaderberg2016reinforcement,silver2017predictron}. Given our setup, both prediction learners and the behavior learner would do well to treat their respective learning tasks as non-stationary and track rather than converge \citep{sutton2007role}, as long as $\xi^2_{t,i}$ is greater than zero. Each sample $\target_{k,i} \sim \thetai$, and $\mui$ is bounded between $[-50,50]$, and $\mui$ is updated on each step $t$ regardless of which action is selected.
Our formalism is summarized in Figure \ref{fig:diagram}.
\begin{figure}[htbp]
      \centering
      \includegraphics[width=0.8\linewidth]{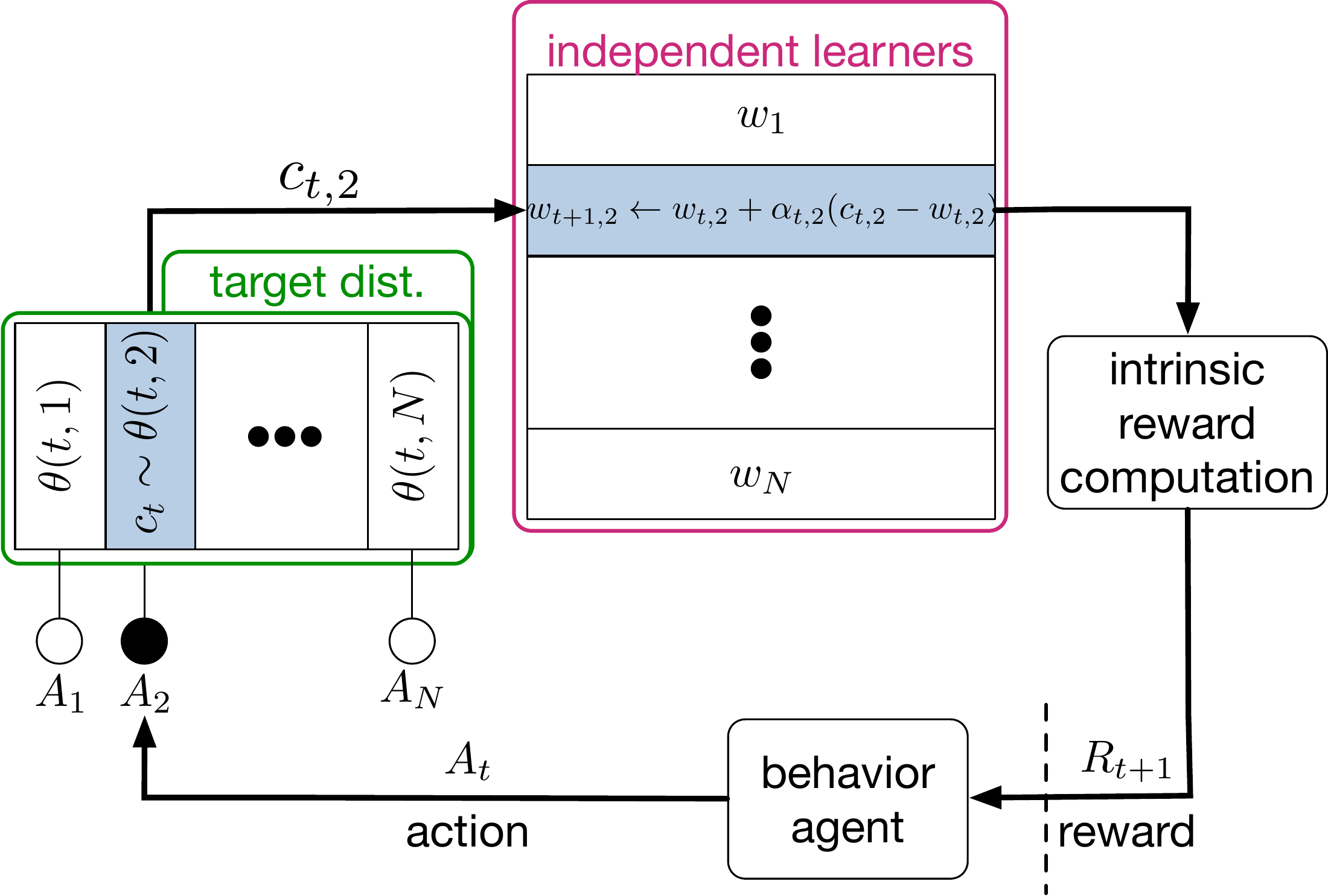}
      \caption{Our parallel multi-prediction learning formulation.}
      \label{fig:diagram}
\end{figure}

\subsection{Non-stationary Bandit Algorithms for Prediction Learning}\label{sec_banditalgs}

We do not focus on the bandit formalism itself nor bandit algorithms in this work. Rather, our goal is to investigate intrinsic rewards and their utility for learning multiple predictions, in the simplest setting in which we can obtain meaningful insights: a bandit-like setting. Our choice of bandit algorithm, therefore, is simply to facilitate this investigation, rather than investigate the properties of the bandit algorithms themselves. We use two different bandit algorithms---a Gradient Bandit and an extension of Dynamic Thompson Sampling (DTS)---so as to ensure our conclusions are not due primarily to the choice of bandit algorithm. We describe these two algorithms below, as well as the reasons for choosing them. 

We cannot simply pick any bandit algorithm, as our prediction learning setting differs from the usual multi-armed bandit setting in at least two ways. First, the distributions of the targets are non-stationary. Non-stationary learning problems have been studied under dynamic bandits, also called restless bandits. The general problem setting is known to be hard, but under some restrictions, some progress can be made. Some algorithms assume piecewise stationarity, such as Discounted UCB or Sliding-Window UCB \citep{garivier2011upper}, or those with a variation budget, which is used to decide how to restart stationary bandit algorithms \citep{besbes2014stochastic}. More suitable for our setting is work assuming restrictions on drift, such as Brownian motion, including State-Oblivious UCB \citep{slivkins2008adapting} and Dynamic Thompson Sampling (DTS) \citep{gupta2011thompson}, or State-Oblivious UCB \citep{slivkins2008adapting} and the Gradient Bandit \citep{sutton2018reinforcement}.

Second, our objective is to minimize error across all learners, but we only see an intrinsic reward corresponding to the target we selected for that step. The need to minimize error across all learners---our second issue---is related to partial monitoring (see \citet{lattimore2019cleaning} for an overview). In partial monitoring, the learning system only receives limited feedback about the true loss incurred. For our prediction setting, the true loss is the MSE over all the predictions. The feedback is only about the prediction for the action selected, and, depending on the intrinsic reward, it is a noisy and indirect measure of the MSE for that prediction. Partial monitoring encompasses a wide range of problems, but it is difficult to develop algorithms and regret bounds for the fully general setting. 
To the best of our knowledge, current algorithms rely on finite outcomes and involve estimating distributions over outcomes. Given the difficulty of even that restricted setting, the additional complication of non-stationarity does not as yet seem to have been tackled. 

Fortunately, for our prediction setting, the structure of our problem (described in Section \ref{prob}) admits a simple approach that performs well in practice: to err on the side of taking an action periodically. Selecting any action is not detrimental, as it provides information about one of the targets. Particularly in a non-stationary setting, each action should be taken periodically, to check if expected reward estimates remain accurate. One reasonable strategy is to obtain a distribution over the actions---not find the single best action---and sample proportionally to that distribution, as is done by the Gradient Bandit. For DTS, we prevent the variance for the Bayesian estimate for each action from dropping below a minimum level, both to account for non-stationarity and to increase the probability that an action will be selected. We find these simple choices to be sufficient for reasonable behavior in our multi-prediction problem setting. 
 
We now describe these two bandit algorithms. 
The Gradient Bandit, specified in \citet[Section 2.8]{sutton2018reinforcement}, attempts to maximize the expected average reward by modifying a vector of action preferences $h \in \mathbb{R}^N$---indexed by action---based on the difference between the reward and average reward baseline: 
\begin{align*}
h_{t+1}(a) \leftarrow \left\{ \begin{array}{ll}
         h_t(a) + \alpha(R_{t+1} - \bar{r})(1-\pi_t(a)) & \mbox{if $A_t = a$};\\
        h_{t}(a) - \alpha(R_{t+1} - \bar{r})\pi_t(a) & \mbox{otherwise}.\end{array} \right.
\end{align*}
where $\bar{r} \in \mathbb{R}$ is the average of all the rewards up to time $t$, maintained using an exponential average, and $\bar{r}$ and $h_0(a)$ are both initialized to zero.
Actions are selected probabilistically according to a softmax distribution which converts the preferences to probabilities:
\begin{equation*}
  Pr\{A_t=a\} = \pi_t(a) \defeq  \frac{e^{h_t(a)}}{\sum_{b=1}^Ne^{h_t(b)}}
\end{equation*}
The Gradient Bandit will sample all the actions infinitely often, though if an action preference is very low then that action will be rarely taken.
Notice that the Gradient Bandit algorithm is similar to policy gradient methods in reinforcement learning. 

The second non-stationary bandit algorithm we use is Dynamic Thompson Sampling (DTS) \citep{gupta2011thompson}.
The algorithm maintains a posterior distribution over the expectation and variance of the reward for each action, using a Bayesian update. The posterior variance is increased after each update, to account for non-stationarity in the rewards. This ensures that, before the posterior is treated like a prior for the next update, it reflects the uncertainty in that prior information, due to the fact that the environment is non-stationary. Otherwise, the posterior would concentrate over time. The distribution over expected rewards is then used in the standard way in Thompson sampling: an estimate is sampled for each action, and the action with maximal sampled value is executed. 
 
The algorithm we use is an extension of DTS, which was only specified for Bernoulli rewards. We extended the approach to Gaussian rewards. The behavior agent assumes the rewards for each action (indexed by $a$) come from a Gaussian distribution, $\mathcal{N}(\mu_a, \sigma_a^2)$ with unknown mean $\mu_a$ and unknown variance $\sigma_a^2$. The behavior agent maintains Bayesian estimates, meaning it maintains a normal-inverse-gamma (NIG) distribution over $(\mu_a, \sigma_a^2)$, which is the conjugate prior for a Gaussian with unknown mean and unknown variance. For each action, we maintain three parameters for a NIG: $(m_a, v_a, n_a)$, where $m_a$ is an estimate of the mean, $n_a$ maintains a count and $v_a$ is an estimate of $n_a$ times the variance.\footnote{A NIG typically has four parameters. For us, the parameter typically called $\alpha$, which is used to normalize $v_a$, exactly equals $n_a/2$, so we do not maintain it explicitly. It is only used to sample the action, using inverse-gamma parameters $(\alpha, v_a) = (n_a/2, v_a)$.} To sample a mean $\mu_a$ for each action, to use for action selection, you first sample $\sigma_a^2$ from an inverse-gamma with parameters $(n_a/2, v_a)$ and then sample $\mu_a$ from $\mathcal{N}(m_a, \sigma_a^2/n_a)$. After picking action $A_t = \tilde{a}$ with the largest $\mu_a$, the behavior agent receives a reward $r$ for taking that action and updates the estimate for $\tilde{a}$. These parameters are updated with the standard Bayesian update:
\begin{align*}
v_{\tilde{a}} &\gets v_{\tilde{a}} + \frac{n_{\tilde{a}}}{n_{\tilde{a}}+1} \frac{(r - m_{\tilde{a}})^2}{2}\\
m_{\tilde{a}} &\gets \frac{n_{\tilde{a}} m_a + r}{n_{\tilde{a}}+1} \\
n_{\tilde{a}} &\gets n_{\tilde{a}}+1
\end{align*}
in the order specified above. If the problem was stationary, then this is the complete update. 

But, the problem is non-stationary, due to the fact that the rewards can change over time. Notice that the posterior variance for $\mu_a$, which is $\frac{v_a}{n_a^2/2 - n_a}$ for this NIG, would gradually shrink to zero as the count $n_a$ increases. To account for non-stationarity, the simple idea behind DTS is to increase this posterior variance after the update in such a way as to minimally impact the mean. For an NIG, this means that we would modify the count and variance parameter for the selected action using
\begin{align*}
v_{\tilde{a}} &\gets \max((1-\alpha) v_{\tilde{a}}, 10^{-2})\\
n_{\tilde{a}} &\gets (1-\alpha)n_{\tilde{a}}
\end{align*}
for decay $\alpha \in (0, 1)$, that behaves like a step-size parameter. The count is decayed by $1-\alpha$, providing an exponential decay on older samples and providing an upper bound on $n_a$ of $\sum_{i=0}^\infty (1-\alpha)^i = 1/\alpha$. The $v_{\tilde{a}}$ is similarly decayed. The max with $10^{-2}$ is to ensure $v_a$ never goes to zero, and so that the variance remains at a minimal level. The new variance of $\mu_a$, according to the NIG after decaying $v^{\text{old}}_a$ and $n^{\text{old}}_a$, is a strict increase 
\begin{align*}
\frac{v_a}{n_a^2/2 - n_a} = \frac{(1-\alpha)v^{\text{old}}_a}{(1-\alpha) ((1-\alpha) (n^{\text{old}}_a)^2/2 - n^{\text{old}}_a)} 
&=  \frac{v^{\text{old}}_a}{(1-\alpha) (n^{\text{old}}_a)^2/2 - n^{\text{old}}_a}\\
&>  \frac{v^{\text{old}}_a}{(n^{\text{old}}_a)^2/2 - n^{\text{old}}_a}
\end{align*}
where we assume $n^{\text{old}}_a > 2$. The mean value $m_a$, though, remains unchanged when we increase the posterior variance. During the standard update to $m_a$ above, however, notice that is resembles an exponential moving average because older values are multiplied by $n_a$, using $n_a m_a$. The algorithm requires an initial mean estimate $m_a = m_0$---a good choice being a large positive value for $m_0$ to encourage exploration---with the initial estimate $v_a = m_0^2$. 
  
\section{Simulating Parallel Prediction Problems} \label{sec:problems}

We consider several prediction problems corresponding to different settings of $\xi^2_{t,i}$ and $\sigma^2_{t,i}$ to define task distribution $\thetai$ in Equation \eqref{drift}. We introduce three problems, with target data simulated from those problems show in Figure \ref{fig:cumulants}.

The {\em Drifter-Distractor} problem has four targets, one for each action:
(1) two (stationary) high-variance targets as \emph{distractors} (2) a slowly \emph{drifting} target and (3) a \emph{constant} target, with $\xi^2_{t,i}$ and $\sigma^2_{t,i}$ for each of these types in Table \ref{tab:switch}. 
A distractor target is simply a noisy, stationary target: the variance is high enough such that an agent might oversample the target even after the mean estimate is accurate. This is inspired by the noisy TV problem \citep{schmidhuber2008driven}.


The {\em Switched Drifter-Distractor} problem is similar to Drifter-Distractor except, after 50,000 time-steps the associations between the actions and the target distributions are permuted as detailed in Table \ref{tab:switch}.
To do well in this problem, the learning system must be able to respond to changes. In addition, in phase two of this problem, two targets exhibit the same drift characteristics; the behavior agent should prefer both actions equally. 

The {\em Jumpy Eight-Action} problem is designed to require sampling different prediction tasks with different frequencies.
In this problem all the $\thetai$ drift, but at different rates and with different amounts of sampling variance as summarized in Table \ref{tab:mbandit}. The best approach is to select several actions probabilistically depending on their drift and sampling variance. 
We add an additional target type, that drifts more dramatically over time, with periodic shifts in the mean:
\begin{equation} 
\mu_{t+1,6} \leftarrow \Pi_{[-50,50]}\left(\mu_{t,6} + \zeta_t \text{Bernoulli}(0.005)\mathcal{N}(10,1.0) \right) \label{eq_jumpy}
\end{equation}
where indicator $\zeta_0 = 1$ and $\zeta_t \in \{-1,1\}$ switches signs if $|\mu_{t+1}| > 50$.  
The sample from a Bernoulli ensures the jumps are rare, but the large mean of the Gaussian makes it likely for this jump to be large when it occurs, as shown in Figure \ref{fig:cumulants}.
This problem simulates a prediction problem where the target changes by a large magnitude in a semi-regular pattern, but then remains constant. This could occur due to changes in the world outside the prediction learner's control and representational abilities. 
\begin{figure}[h!]
      \centering
      \includegraphics[width=1.0\linewidth]{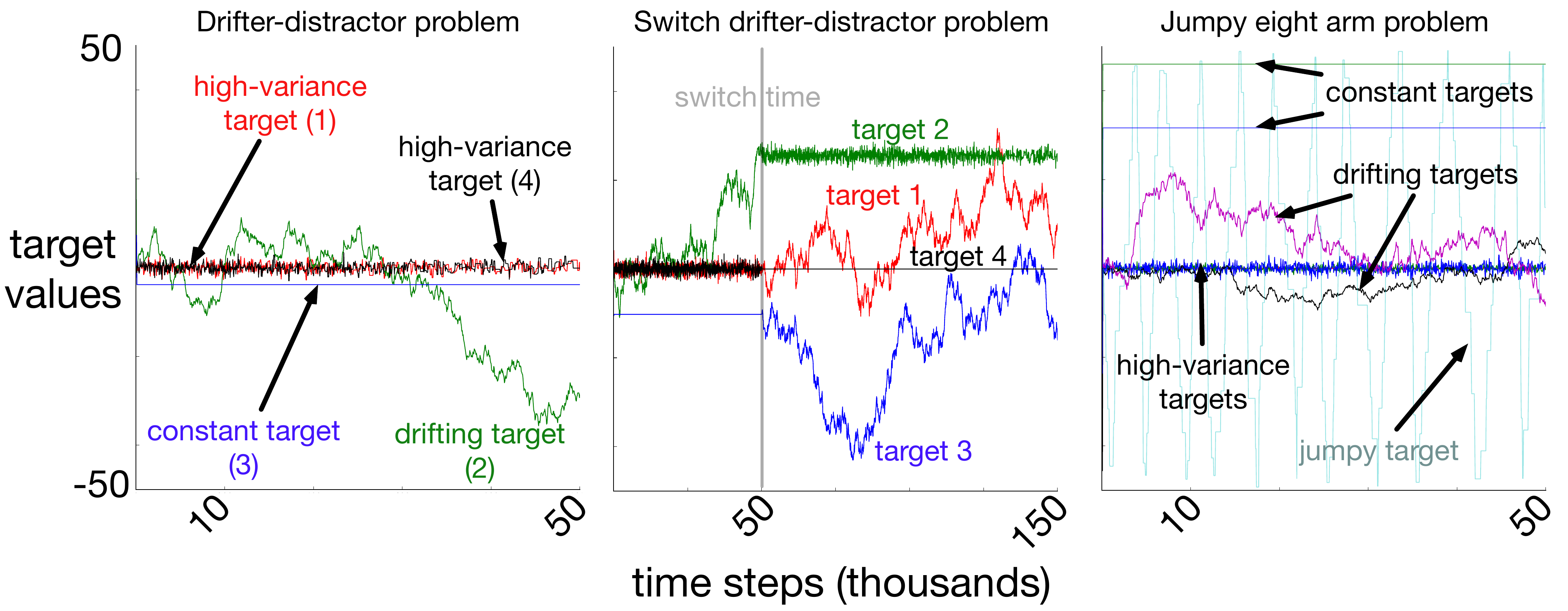}
      \caption{Each subplot shows the target data generated by one run of the problem, with Drifter-Distractor (left), Switched Drifter-Distractor (middle), and Jumpy Eight-Action Problem (right). 
      }
      \label{fig:cumulants}
\end{figure}
\begin{table}[h!]
\centering
\begin{tabular}{lccc}  
\toprule
 \emph{target type} & $\mu_0$ & $\sigma^2$& $\xi^2$ \\  \midrule
 constant  & uniform(-50,50) & 0 & 0  \\ 
 distractor & 0 & 1 & 0\\ 
 drifter & 0 &0& 0.1
 \\\bottomrule
\end{tabular}
    \caption{These parameters define each target distribution used in the {\em Drifter-Distractor} and the {\em Switched Drifter-Distractor} problems. The parameter $\mu_0$ specifies the initial mean of each target, $\sigma^2$ is the sampling variance, and $\xi^2$ is the drift variance.}     
 \label{tab:DD}
\end{table}
\begin{table}[h!]
\centering
\begin{tabular}{lccc}  
\toprule
\emph{target} &\emph{phase 1}& $\rightarrow$ & \emph{phase 2}\\  \midrule
 target 1  & distractor & $\rightarrow$ &drifter  \\ 
 target 2 & drifter & $\rightarrow$ &distractor\\ 
 target 3 & constant & $\rightarrow$ &drifter \\
 target 4 & distractor & $\rightarrow$ &constant 
 \\\bottomrule
\end{tabular}
    \caption{The target distributions in the {\em Switched Drifter-Distractor} change part way through the task. Phase one lasts for 50,000 time steps, then targets are permuted and remain fixed for the remainder of the experiment (another 100,000 steps). The initial parameters for each target type---constant, distractor and drifter---are the same as in the Drifter-Distractor Problem described in Table \ref{tab:DD}}
    \label{tab:switch}
\end{table}

\begin{table}[h!]
\centering
\begin{tabular}{lcccccc}  
\toprule
 Task  & 1  & 2  & 3  & 4 & 5 &  7 \& 8  \\ 
  \midrule
 $\sigma^2$& 0.1 & 0.5 & 1.0 & 0.01 & 0.01 & 0.0\\ 
 $\xi^2$ & 0.0 & 0.0 & 0.0 & 0.01 & 0.05 & 0.0 \\\bottomrule
\end{tabular}
    \caption{Parameters defining $\thetai$ for each prediction task in the {\em Jumpy Eight-Action} problem, where $\sigma^2$ is the sampling variance and $\xi^2$ is the drift variance for Equation \eqref{drift}. \textbf{Prediction Task 6} is special, defined in Equation \eqref{eq_jumpy}.}
    \label{tab:mbandit}
\end{table}

\clearpage
\section{Introspective Prediction Learners}\label{isl}

The behavior of a learning system that maximizes intrinsic rewards relies on the underlying prediction learning algorithms as well as the definition of the intrinsic reward. In this section we introduce a distinction between two categories of learners, for which behavior can be substantially different: introspective and non-introspective learners. We consider a learner to be {\bf introspective} if the algorithm can modulate its own learning without help from an external process. More concretely, an introspective learner stops updating if it cannot make progress. For example, in the case of prediction learning, an introspective learner would regulate its updates to mitigate noise in its prediction targets. A {\bf non-introspective learner}, on the other hand, will continually update regardless of learning progress. 

In this paper we consider two basic settings representing non-introspective and introspective learners, used as prediction learners in our multi-prediction problem. We use basic LMS learners with a constant step-size parameter as our non-introspective learner. With a {\em constant step-size} parameter, the LMS algorithm will always try to adapt its estimates toward the sample targets on each time step. It does not matter if the target exhibits high variance---say centered mean zero---or if the target is actually constant; the LMS algorithm will continue to adapt its estimates attempting to track each target in the online setting. Consider how a constant global step-size parameter would work on our Drifter-Distractor Problem discussed above. If the step-size parameter value is too large for the high-variance target, then the prediction learner will continually make large updates due to the sampling variance, never converging to low error. If the step-size parameter is too small for the tracking target, then the prediction learner's estimate will often lag, causing high-error. A constant global step-size parameter cannot balance the need to track the drifting targets, and the need to learn slowly on the high-variance targets. 

To create a simple introspective learner for our setting, we simply combine our LMS predictors with a step-size adaption method called Autostep. Autostep is a simple meta-learning algorithm that adapts the step-size parameter of each LMS learner over-time \citep{mahmood2012tuning}. The basic idea behind Autostep is to increase the step-size parameter when learning is progressing, and lower the step-size parameter value when learning is not progressing. It does so by keeping a trace, $h\in\mathbb{R}$, of the previous prediction errors. Roughly speaking, if the error changes sign often then the predictions are not improving and the step-size parameter value should be lowered. If the error is mostly of the same sign, then the step-size parameter value should not be reduced. Autostep has one key hyper-parameter, the meta learning-rate: this controls how quickly the algorithm changes the step-size parameter ($\alpha$). The full pseudocode, specialized to our stateless tracking tasks, is given below. Note that Autostep changes the step-size parameter with a multiplicative exponential, which allows geometric or rapid changes to the LMS learners step-size parameter.  

\begin{algorithm}
\caption{: {\em The Autostep algorithm specialized to stateless prediction} \\ $\kappa$ is the meta learning-rate parameter \\ $n$ and $h$ are scalar memory variables initialized to 1 and 0 \\$\delta$ is the prediction error and $\alpha$ (initialized to 1.0) the step-size parameter of predictor $i$}\label{auto}
\begin{algorithmic}[1]
\Procedure{Autostep}{$\delta$}
\State $n \gets \max(|\delta h|, n+\frac{1}{10000}\alpha(|\delta h| - n))$
\State $\alpha \gets \min(\alpha \exp(\kappa \frac{\delta h}{n}),~0.5)$
\State $h \gets h (1-\alpha) + \alpha \delta $
\EndProcedure
\end{algorithmic}
\end{algorithm}

To give some intuition about how Autostep changes the step-size parameter, consider what happens when we apply it to the Drifter-Distractor Problem in Figure \ref{fig:alphas}. Here we simply plot $\alpha$ over time for four LMS learners---one for each target---with each step-size parameter adapted by Autostep. We used the Gradient Bandit and Weight Change reward~\footnote{The details of the intrinisic reward function used to generate the data do not matter for the purpose of the this experiment. Nevertheless, the Weight Change reward will be defined in the next section.} to generate the behavior. The initial $\alpha$ of each LMS learner were set to one. The lines for the constant target (blue) and drifter target (green) are overlapping, and the lines for the high-variance targets (red and black) are overlapping. Autostep progressively decreases $\alpha$ for the high-variance targets, as the updates oscillate around zero. The update magnitude (or error) for the constant target goes to zero, and so Autostep stops changing $\alpha$. This makes sense: why change the $\alpha$ if the prediction is perfect. Autostep keeps the $\alpha$ for the learner estimating the drifting target high, because continual progress is possible. On each time step the LMS learner moves its estimate towards the recent sample and most of these updates are in the same direction, at least over a recent window of time. In terms of prediction performance, Autostep significantly improves tracking, enabling different update rates for different prediction learners and reducing $\alpha$ on unlearnable targets or noisy targets once learning is complete---as you will see in our main experiments below. 

\begin{figure}[h!]
      \centering
      \includegraphics[width=0.4\linewidth]{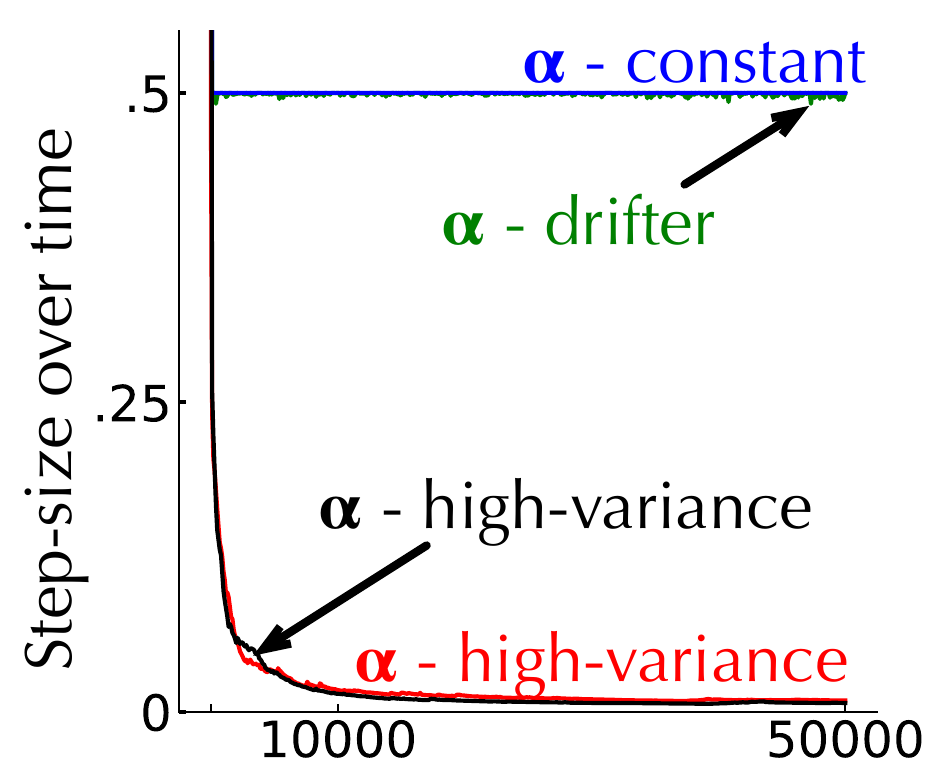}
      \caption{Sample run showing how Autostep changes the step-size parameters ($\alpha$) over time with {\em Weight Change} reward. The lines for the constant target (blue) and drifter target (green) are overlapping, and the lines for the high-variance targets (red and black) are overlapping. }
      \label{fig:alphas}
\end{figure}

We experimented with other step-size adaption methods, including AdaDelta and RMSProp, but the results were qualitatively similar. In this study we chose Autostep because (a) it was specifically designed for non-stationary, incremental, online tracking tasks like ours, (b) it uses a simple and easy to interpret update rule, and (c) there is a long literature demonstrating the practical utility of this method dating back to it's origins in the IDBD method \citep{sutton1992adapting}.

The choice of using meta-learning to obtain introspective learners not only works well in our multi-prediction tasks, but also should scale to larger tasks with function approximation in future work. 
Step-size adaption methods like Adam and RMSProp can speed up training in neural networks and make learning more robust to non-stationarity. In online reinforcement learning, extensions of Autostep have be shown to improve prediction and control performance with function approximation \citep{kearney2018tidbd,guenther2020examining}. We discuss these extensions, and how our results go beyond stateless tracking at the end of the paper. 
It is worth emphasizing that introspective learners are not optimal learners and that they are in fact the most common type of agent used in deep reinforcement learning. 
The main criterion is that an introspective agent should regulate it's own updates based on an internal measure of learning progress. 

\section{Intrinsic Rewards for Multi-prediction Learning}

Many learning systems draw inspiration from the exploratory behavior of humans and animals, uncertainty reduction in active learning, and information theory---and the resulting techniques could all be packed into the suitcase of curiosity and intrinsic motivation. In an attempt to distill the key ideas and perform a meaningful yet inclusive empirical study, we consider only methods applicable to our problem formulation of multi-prediction learning.
Although few approaches have been suggested for off-policy multi-task reinforcement learning---approaches by \citet{chentanez2005intrinsically,white2014surprise} as notable exceptions---many existing approaches can be used to generate intrinsic rewards for multiple, independent prediction learners (see \citeauthor{barto2013intrinsic}'s (\citeyear{barto2013intrinsic}) excellent summary). We first summarize methods we evaluate in our empirical study.
The specific form of each intrinsic reward discussed below is given in Table \ref{tab:rewards}, with italicized names below corresponding to the entries in the table. We conclude by mentioning several rewards we did not evaluate, and why. 

Several intrinsic rewards are based on \textbf{violated expectations}, or surprise. This notion can be formalized using the prediction error itself to compute the instantaneous {\em Absolute Error} or {\em Squared Error}. We can obtain a less noisy measure of violated expectations with a windowed average of the error, which we call {\em Expected Error}. Regardless of the specific form, if the error increases, then the intrinsic reward increases encouraging further sampling for that target. Such errors can be normalized, such as was done for {\em Unexpected Demon Error} \citep{white2014surprise}, to mitigate the impact of noise in and magnitude of the targets. 

Another category of methods focus on \textbf{learning progress}, and assume that the learning system is capable of continually improving its policy or predictions. This is trivially true for approaches designed for tabular stationary problems \citep{chentanez2005intrinsically,still2012information,little2013learning,meuleau1999exploration,barto2005intrinsic,szita2008many,lopes2012exploration,schossau2016information}. 
The most well-known approaches for integrating intrinsic motivation make use of rewards based on improvements in (model) error: including {\em Error Reduction} \citep{schmidhuber1991possibility,schmidhuber2008driven}, and model {\em Error Derivative} approach \citep{oudeyer2007intrinsic}. Improvement in the value function can also be used to construct rewards, and can be computed from the {\em Positive Error Part} \citep{schembri2007evolving}, or by tracking improvement in the value function over all states \citep{barto2005intrinsic}. 
As our experiments reveal, however, intrinsic rewards requiring improvement can lead to less desireable behavior in non-stationary tracking problems.

An alternative to learning progress is to reward \textbf{amount of learning}. 
This does not penalize errors becoming worse, and instead only measures that estimates are changing: the prediction learner is still adjusting its estimates and so is still learning. {\em Bayesian Surprise} \citep{itti2006bayesian} formalizes the idea of amount of learning. For a Bayesian learner, which maintains a distribution over the weights, Bayesian Surprise corresponds to the KL-divergence between this distribution over parameters before and after the update. This KL-divergence measures how much the distribution over parameters has changed. Bayesian Surprise can be seen as a stochastic sample of Mutual Information, which is the expected KL-divergence between prior and posterior across possible observed targets. We discuss this more in Section \ref{sec_bayes}. Other measures based on Information Gain have been explored \citep{still2012information,little2013learning,achiam2017surprise,de2018curiosity,berseth2019smirl}. In the tabular case different variations of information-gain reward perform similarly to Bayesian Surprise empirically \citep{little2013learning}. 

Though derived assuming stationarity and Bayesian learners, we provide an approach to approximate Bayesian Surprise for our non-stationary setting with non-Bayesian learners. The prediction learner's main objective is to estimate an unknown mean. A Bayesian learner maintains a distribution over this unknown mean, based on the chosen distribution for the targets. A simple choice is to use a Gaussian distribution for the targets, with an unknown mean but a known variance, giving a Gaussian conjugate prior. The variance is not actually known; we maintain an estimate $v^{(y)}_{t}$ of the variance of the target $\text{Var}[C_{t,i}]$.\footnote{
Typically, a Bayesian learner would simply maintain a distribution over both the mean and variance, if they are both unknown. Our goal here, though, is to approximate Bayesian surprise for a non-Bayesian learner. Since the learner only estimates the mean, we assume that the corresponding Bayesian learning can only maintain a distribution over the mean.} The posterior uses the learner's mean estimate and the posterior variance for a Bayesian update, which is proportional to $v^{(y)}_{t}/t$. Effectively, the approximate Bayesian surprise is maintaining a posterior, but uses the learners mean estimate instead of its own. To additionally account for non-stationarity, we use the same idea behind DTS: increasing the posterior variance after each update (see Table \ref{tab:rewards} for the formula).\footnote{The count $n$ is decayed by $1-\beta$, and the variance is prevented from decreasing below $10^{-2}$. To keep the update simpler, we assume the initial variance for the prior is very large---which is in fact reasonable as it implies no prior knowledge about the unknown mean. Therefore the initial variance can be omitted in the posterior variance, as it has a negligible affect.} We make no claims that this is the ideal strategy to approximate to Bayesian Surprise for non-Bayesian learners in non-stationary problems; it is rather a reasonable simple strategy in an effort to include it as best as possible in our experiments.

We can additionally consider non-Bayesian strategies for measuring amount of learning, including those based on change in error ({\em Error Derivative}), {\em Variance of Prediction}, {\em Uncertainty Change}---how much the variance in the prediction changes---and the {\em Weight Change}, which we discuss in more depth in the next section. Note that several learning progress measures can be modified to reflect amount of learning by taking the absolute value, and so removing the focus on increase rather than change (this must be done with care as we likely do not want to reward model predictions becoming worse, for example). 

\renewcommand{\arraystretch}{1.42}
\begin{figure}
\begin{longtable}{p{\dimexpr0.45\linewidth-\tabcolsep\relax}p{\dimexpr0.55\linewidth-\tabcolsep\relax}}
\label{tab:rewards}\\

\textbf{Reward Name} & $R_{t,i}$ \\
\hline \midrule
\endfirsthead

%


 \makecell{\bf Error Derivative \\ \citep{oudeyer2007intrinsic}} & $ \displaystyle \left\lvert  \frac{1}{\eta+1}\sum_{j=0}^\eta \delta_{t - j-\tau,i}^2 - \frac{1}{\eta+1}\sum_{j=0}^\eta \delta_{t-j,i}^2 \right\rvert$ \\ \nopagebreak
 \multicolumn{2}{p{\dimexpr\linewidth-2\tabcolsep\relax}}{$\tau\le\eta<t$, where $\eta$ specifies the length of the window and $\tau$ the amount of overlap}\\  
 \makecell[b]{{\bf Expected Error}} & $\left\lvert \overline{\delta_{t,i}}^\beta \right\rvert$ \\
 \multicolumn{2}{p{\dimexpr\linewidth-2\tabcolsep\relax}}{$\overline{x_t}^\beta$ denotes the exponentially weighted average of $x_0$ to $x_t$ with with decay rate $1-\beta$} \\ 
 {\bf Step-size Change} & $|\alpha_{t-1,i} - \alpha_{t,i}|$ \\ \addlinespace[1ex]
\makecell{{\bf Error Reduction} \\ \citep{schmidhuber1991curious}} & $ |\delta_{t-1,i}| - |\delta_{t,i}|$ \\ \nopagebreak
 \makecell{{\bf Squared Error} \\ \citep{gordon2011reinforcement}} & $\delta_{t,i}^2$ \\
 \makecell{{\bf Bayesian Surprise} \\ \citep{itti2006bayesian}}   & $ \displaystyle \log_2\left(\frac{v_{t,i}}{v_{t-1,i}}\right) +\frac{v_{t-1,i} + (\that_{t-1,i} - \that_{t,i})^2}{2v_{t,i}}-\frac{1}{2}$ \\ 
  \multicolumn{2}{p{\dimexpr\linewidth-2\tabcolsep\relax}}{$v_{t,i} = \max(v^{(y)}_t/n_t, 10^{-2})$ where $n_t = (1-\beta) n_{t-1} + 1$ and $v^{(y)}_t$ is an estimate of $\text{Var}[C_{t,i}]$, using an exponential average variant of Welford's algorithm, with $v_{t,i} = (1-\beta) v_{t-1,i} + \beta (\target_{t,i} - \that_{t-1,i})(\target_{t,i} - \that_{t,i})$ for $0 < \beta < 1$}\\ \addlinespace[1ex]
\makecell{{\bf Unexpected Demon Error} \\ \citep{white2014surprise,white2015developing}} & $ \displaystyle \left\lvert \frac{\overline{\delta_{t,i}}^\beta}{\sqrt{{\text{Var}}[\delta_{i}]} + \epsilon} \right\rvert $ \\ 
\multicolumn{2}{p{\dimexpr\linewidth-2\tabcolsep\relax}}{$\epsilon$ is a small constant set to $10^{-6}$ in our experiments}\\
 {\bf Uncertainty Change} & $ \left\lvert \text{Var}[\that_{t-1,i}] -\text{Var}[\that_{t,i}] \right\rvert$ \\
{\bf Variance of Prediction} & $ \text{Var}[\that_{t,i}]$ \\
 {\bf Weight Change} & $\| w_{t,i} - w_{t-1,i} \|_1 = \alpha_t \| \that_{t,i} - \that_{t-1,i}\|_1$ \\ 
 \midrule
 \makecell{{\bf Absolute Error*} \\ \citep{schmidhuber1991possibility}} & $|\delta_{t,i}|$ \\ 
  %
 \makecell{\bf Positive Error Part* \\ \citep{mirolli2013functions}} &  $\max(\delta_{t,i}, 0)$ \\ 
 \bf Variance of Error* & $\text{Var}[\delta_{t,i}]$ \\ 
 %
 \bf Uncertainty Reduction* & $ \text{Var}[\that_{t-1,i}] -\text{Var}[\that_{t,i}]$ \\ 
\hline \midrule
\end{longtable}
\caption{Intrinsic rewards investigated in this work. Separate statistics are maintained for each learning task $i$, and only updated when task $i$ is selected by the behavior agent. Non-starred rewards are included in the results. Starred rewards were tested but performed poorly. }
\end{figure}

There are several strategies which we omit, because they would (1) result in uniform exploration in our pure exploration problem, (2) require particular predictions about state to drive exploration, (3) are designed for the offline batch setting, or (4) are based on statistics of the targets rather than the statistics generated by the prediction learners. Count-based approaches \citep{brafman2002r,bellemare2016unifying,sutton2018reinforcement} are completely unsupervised, rewarding visits to under-sampled states or actions---resulting in uniform exploration in our problem. Though count-based approaches are sometimes used in learning systems, they reflect novelty rather than learning progress or surprise \citep{barto2013novelty}. 

The second set of strategies we omit are methods that use a model to encourage exploration \citep{schmidhuber2008driven,chentanez2005intrinsically,stadie2015incentivizing,pathak2017curiosity,pathak2019self} such as by using Bayesian Surprise for next-state prediction \citep{houthooft2016vime}. Subgoal discovery systems \citep{kulkarni2016hierarchical,andrychowicz2017hindsight,pere2018unsupervised} define rewards to reach particular states. Empowerment and state control systems are explicitly designed to respect and use the fact that some tasks or regions of the state-space cannot be well learned. Often such systems use only unsupervised signals relating to statistics of the exploration policy, ignoring the statistics generated by the learning process itself \citep{karl2017unsupervised}. Like count-based approaches, unsupervised measures like this would induce uniform exploration in our stateless task. 

Curriculum learning---learning what task to sample next---is closely related to our multi-prediction problem. Graves {\em et al.} (2017) introduce several measures for batch curriculum learning that are related to the ideas underlying the intrinsic rewards discussed above. Most related, Prediction Gain corresponds to {\em Error Reduction}, albeit assuming a batch of data rather than an online instance. An approximation, called Gradient Prediction Gain, corresponds to the norm of the gradient; for our setting, this is the same as the {\em Absolute Error}. Several of Graves' measures require the ability to sample new batches of data, such as Supervised Prediction Gain and Target Prediction Gain. Finally, Graves {\em et al.} investigated several Complexity Gain measures for the neural networks, measuring KL divergence between the posterior and a learned prior. The prior is updated towards the previous posterior, and so the resulting KL is related to Bayesian surprise. The KL itself, though, is not used: rather, the gain in complexity is measured by looking at the difference in two KLs, before and after an update. These approaches require Bayesian learners with a separate prior distribution to be learned just to measure the complexity. The most simple and computationally feasible of these is L2 Gain, which is simply the difference in $\ell_2$ norm of the weights before and after and update: $\| w_{t,i} \|_2^2 - \| w_{t-1,i} \|_2^2$. This rewards the learning system for making the weights smaller, and performed worse than random for curriculum learning \citep{graves2017automated}.

Finally, we do not test intrinsic rewards based only on targets, such as variance of the target. To see why, consider a behavior that estimates the variance for a constant target, and quickly determines it only needs to select that action a few times. The prediction learner, however, could have a poor estimate of this target, and may need many more samples to converge to the true value. Separately estimating \emph{possible} amount of learning from \emph{actual} amount of learning has clear limitations. 
Note that in the stationary bandit setting, with a simple sample average learner, the variance of the prediction target provides a measure of uncertainty for the learned prediction \citep{audibert2007variance,garivier2011upper,antos2008active}, and has been successfully applied in education applications \citep{liu2014trading,clement2013multi}. When generalizing to other learners and problem settings, however, variance of the target will no longer obviously reflect uncertainty in the predictions. We therefore instead directly test intrinsic rewards that measure uncertainty in predictions, including Uncertainty Change and Variance of Prediction.

\section{Optimal Behavior for Multi-prediction Learning, and Approximations}\label{sec_bayes}

One natural question given this variety of intrinsic rewards, is if there is an optimal approach. In some settings, there is in fact a clear answer. In a stationary, stateless problem where the goal is to estimate means of multiple targets, it has been shown that the behavior agent should take actions proportional to the variance of each target to obtain minimal regret \citep{antos2008active}. For a stationary setting, with state, an optimal approach would be to take actions to maximize Information Gain---the reduction in entropy after an update---across learners \citep{orseau2013universal}. We therefore use Information Gain as the criterion to measure optimal action selection. In this section, we describe how to maximize Information Gain in an ideal case, and how to approximate it otherwise. 
The goal of this section is to provide intuition and motivation, as we do not yet have theoretical claims about the approximation strategies. 

We first show that Information Gain is maximized when maximizing expected Bayesian surprise, assuming Bayesian learners. A Bayesian learner updates weights $w$ for a parameterized distribution $p_w$ on the parameters $\theta$ needed to make the prediction $\that$. The parameters can be seen as a random variable, $\Theta$, with distribution $p_w$. The goal is to narrow this distribution around the true parameters $\theta^*$ that generate $\target$---that is, $c$ is sampled from $p(c | \theta^*)$. After seeing each new sample, the posterior distribution over parameters is computed using the previous distribution $p_{w_t}(\theta)$ and the new sample, $\target_t$, using the update 
\begin{equation*}
p_{w_{t+1}}(\theta) \defeq p_{w_{t}}(\theta | \target_t) = \frac{p(\target_t | \theta) p_{w_t}(\theta)}{ p_{w_t}(\target_t)}
.
\end{equation*}
The term in the denominator is dependent on $w_t$ because $p_{w_t}(\target_t) = \int p(\target_t | \theta) p_{w_t}(\theta) d\theta$. 
A \emph{Bayesian learner} is one that uses exact updates to obtain the posterior. We assume the prior is appropriately specified so that $p_{w_t}(\target) \neq 0$, and so $p(\theta | \target_1, \ldots, \target_n)$ has non-zero support as $n \rightarrow \infty$ almost surely for any stochastic sequence $\target_1, \ldots, \target_n$. 

\emph{Bayesian surprise} is defined as the KL divergence between the distribution over parameters before and after an update \citep{itti2006bayesian}
\begin{equation}
\text{KL}(p_{w_{t+1}} || p_{w_{t}}) = \int p_{w_{t+1}}(\theta) \log\frac{p_{w_{t+1}}(\theta)}{p_{w_{t}}(\theta)} d\theta
.
\label{eq_kl_weights}
\end{equation}
The Bayesian surprise is high when taking an action that produces a stochastic outcome $\target_t$ that results in a large change in the prior and posterior distributions over parameters. The expectation of the KL-divergence over stochastic outcomes, with a Bayesian learner, corresponds to the Information Gain. This result is well-known, but we explicitly show it in the following theorem for completeness. Notice that Information Gain defined in Equation \eqref{eq_ig} is relative to the model class of our learner, rather than some objective notion of information content. 

\begin{theorem} 
Assume targets $\targetr \in \mathcal{C}$ are distributed according to true parameters $\theta^*$, with density $p_{\theta^*}: \mathcal{\targetr} \rightarrow \infty$ and event space . 
For a Bayesian learner, that maintains distribution $p_{w_t}$ over parameters $\Theta$, the mutual information (also called the Information Gain) $I(\targetr, \Theta)$ equals the expected KL-divergence between the posterior and prior
  \begin{equation}
I(\targetr, \Theta) = \mathbb{E}[\text{KL}(p_{w_{t+1}} || p_{w_{t}})] \label{eq_ig}
\end{equation}
where the expectation is over stochastic outcomes $\targetr$ that produce $w_{t+1}$ from $w_t$. 
\end{theorem}
\begin{proof}
\begin{align*}
I(\targetr, \Theta) &= \int \int p_{\theta^*, w_{t}}(\target, \theta) \log \frac{p_{\theta^*, w_{t}}(\target,\theta)}{p_{\theta^*}(\target) p_{w_{t}}(\theta)} d\target d\theta && \triangleright \ p_{\theta^*, w_{t}}(c, \theta) = p_{\theta^*}(\target) p_{w_{t}}(\theta | \target) \\
&= \int \int p_{\theta^*}(\target) p_{w_{t}}(\theta | \target) \log \frac{p_{w_{t}}(\theta | \target)}{p_{w_{t}}(\theta)} d\target d\theta && \triangleright \ p_{\theta^*}(\target) \text{ cancels in the fraction}\\
&= \int  p_{\theta^*}(\target) \left[\int p_{w_{t}}(\theta | \target)  \log \frac{p_{w_{t}}(\theta | \target)}{p_{w_{t}}(\theta)} d\theta \right] d\target && \triangleright \ p_{\theta^*}(\target) \text{ does not involve $\theta$}\\
&= \int  p_{\theta^*}(\target) \left[ \int  p_{w_{t+1}}(\theta) \log \frac{p_{w_{t+1}}(\theta)}{p_{w_{t}}(\theta)} d\theta \right] d\target && \triangleright  \ p_{w_{t+1}}(\theta) = p_{w_{t}}(\theta | \target) \\
&= \int p_{\theta^*}(\target) \text{KL}(p_{w_{t+1}} || p_{w_{t}}) d\target && \triangleright  \ \text{Equation \eqref{eq_kl_weights}}\\
&= \mathbb{E}[\text{KL}(p_{w_{t+1}} || p_{w_{t}})]
\end{align*}
The weights $w_{t+1}$ are dependent on the observed $\target$. By definition, this integral gives an expected KL, across possible observed $\target$. 
\end{proof}

To make this more concrete, consider Bayesian surprise for a Bayesian learner with a simple Gaussian distribution over parameters. 
For our simplified problem setting, the weights for the Bayesian learner are $w_t = (\mu_t, \sigma_t^2)$ for the Gaussian distribution over the parameters $\theta$, which in this case is the current estimate of the mean of the target, $\that_t$. 
The Bayesian surprise is
\begin{equation*}
\text{KL}(p_{w_{t+1}} || p_{w_{t}}) = \log \frac{\sigma_{t+1}^2}{\sigma_t^2} +  \frac{\sigma_{t}^2 + (\mu_t - \mu_{t+1})^2}{2 \sigma_{t+1}^2} - \frac{1}{2}
.
\end{equation*}
We can make this even simpler if we consider the variance $\sigma^2$ to be fixed, rather than learned. The Bayesian surprise then simplifies to
 %
\begin{align}
\text{KL}(p_{w_{t+1}} || p_{w_{t}}) 
&= \log \frac{\sigma^2}{\sigma^2} + \frac{\sigma^2 + (w_t - w_{t+1})^2}{2 \sigma^2} - \frac{1}{2} \nonumber\\
&= 0 + \frac{\sigma^2}{2 \sigma^2} + \frac{(w_t - w_{t+1})^2}{2 \sigma^2} - \frac{1}{2} \nonumber\\
&=  \frac{(w_t - w_{t+1})^2}{2 \sigma^2} \label{eq_kl_gauss}
.
\end{align}
This value is maximized when the squared change in weights $(w_t - w_{t+1})^2$ is maximal. Therefore, though Bayesian surprise in general may be expensive to compute, for some settings it is as straightforward as measuring the change in weights. 

Additionally, we can also consider approximations to Bayesian surprise for non-Bayesian learners. A non-Bayesian learner typically estimates the parameters $\theta_t$ directly, such as by maximizing likelihood or taking the maximum a posteriori (MAP) estimate
\begin{equation*}
\theta_{t+1} \defeq \arg\max_{\theta} p(\theta | \target_1, \ldots, \target_t)
.
\end{equation*}
Now instead of maintaining the full posterior $p(\theta | \target_1, \ldots, \target_t)$ as $p_{w_{t+1}}$, the prediction learner need only learn $\theta_{t+1}$ directly. Because $\theta_{t+1}$ is the mode of the posterior, for many distributions $\theta_{t+1}$ will actually equal a component of $w_{t+1}$. For the Gaussian example above with a learned variance, $\theta_{t+1}$ equals the first component of $w_{t+1}$, the mean $\mu_{t+1}$. For a fixed variance, $\theta_{t+1}$ exactly equals $w_{t+1}$. Therefore, the non-Bayesian learner would have the exact same Information Gain, measured by the Bayesian surprise in \eqref{eq_kl_gauss}. Note that MAP makes the same probabilistic assumptions for the posterior as the Bayesian learner, but is considered non-Bayesian here because it does not maintain the parameters of this posterior. 

This direct connection, for Bayesian and non-Bayesian learners, only exists for a limited set of distributions. One such class is the natural exponential family distribution over the parameters. Examples include the Gaussian with fixed variance and mean $w_{t}$ and the Gamma distribution with a fixed shape parameter and scale parameter $w_t$.  Each natural exponential family has the property that the KL-divergence between two distributions with parameters $w_t$ and $w_{t+1}$ corresponds to a (Bregman) divergence directly on the parameters \citep{banerjee2005clustering}. For a Gaussian, this divergence is the squared error normalized by the variance, as above in Equation \eqref{eq_kl_gauss}. Another distribution that has this connection is a Laplace distribution with mean $w_t$ and fixed variance $2b^2$. Then the KL-divergence is  $KL(p_{w_{t+1}} || p_{w_t}) = | w_t - w_{t+1} | / b$. 

This connection is limited to certain posterior distributions, but is true for general problem settings, even the general reinforcement learning setting. The distributions before and after an update, $p_{w_t}$ and $p_{w_{t+1}}$ respectively, are over the parameters of the prediction learner. These parameters are more complex in settings with state---such as parameters to a neural network---but we can nonetheless consider exponential family distributions on those parameters.


This discussion motivates a simple proposal to approximate Bayesian surprise and Bayesian learners for a general setting with non-Bayesian learners: using weight change with \emph{introspective} learners. 
An introspective learner is not a precise definition, but rather a scale. A perfectly introspective learner would be a Bayesian learner, or in some cases the corresponding MAP learner. A perfectly non-introspective learner could be a random update. The more closely the learner approximates the weights to the perfectly introspective learner, the better its solution and the better the Bayesian surprise reflects the Information Gain. Further, because the underlying distribution may not be known, we use the change in weights as an approximation.


For concreteness, consider the following learning system. Each prediction learner is augmented with a procedure to automatically adapt the step-size parameter $\alpha_{k,i}$, based on the errors produced over time ($\delta_{i,0:k}$). In this paper we use the Autostep algorithm \citep{mahmood2012tuning} described in Section \ref{isl}. Recall, that the Autostep algorithm automatically reduces $\alpha_{k,i}$ towards zero if the target is unlearnable, increase $\alpha_{k,i}$ when successive errors have the same sign, and does not change $\alpha_{k,i}$ if the error is zero. We call a learner with a fixed step-size parameter, on the other hand, non-introspective, because the learner will forever chase the noise. The weight change for such a learner would not be reflective of Information Gain, reflecting instead only the inadequacy of the learner. A learner equipped with Autostep, on the other hand, like a Bayesian or MAP learner, will stop learning once new samples provide no new information. 

This proposal reflects a simply reductionist philosophy: there should be an explicit separation in the role of the behavior agent and the role of the prediction learners. The behavior agent should balance data generation amongst parallel prediction learners. The prediction learner's primary responsibility is to estimate their target accurately. If the behavior agent trusts that the prediction learners are using the data appropriately, then the learning system can make use of intrinsic rewards based solely on the prediction learner's parameters, such as the change in the weights. The alternative is to assume that the intrinsic rewards must be computed to overcome poor learning. This approach would require the learning system to recognize when a prediction learner is non-introspective, and decrease the reward associated with that learner. If the learning system can measure this, though, then presumably so too can the prediction learner---they are after all part of the same system. The learner should then be able to use the same measure to adjust its own learning.

In this work, we define the change in weights using the $\ell_1$ norm, 
\begin{equation}
\text{Weight Change}(w_t, w_{t+1}) \doteq \| w_{t} - w_{t+1} \|_1
.
\end{equation}
In our setting, the Weight Change is simply Absolute Error scaled by the step-size parameter, emphasizing the role that learner capability plays in ensuring an effective reward.
\begin{equation}
\| w_{t} - w_{t+1} \|_1 = \alpha_{t,i} \| \that_{t,i} - \that_{t-1,i}\|_1 = \alpha_{t,i} | \delta_{t,i} |
\end{equation}

\textbf{Remark:} The above discussion applies to the non-stationary setting, by treating the non-stationarity as partial observability. We can assume that the world is stationary, driven by some hidden state, but that it appears non-stationary to the learning system because it only observes partial information. If a Bayesian prediction learner had the correct model class, it could still maximize Information Gain. For example, the prediction learner could know there is a hidden parameter $\xi$ defining the rate of drift for the mean of the distribution over $C$. It could then maintain a posterior over both $\xi$ and the mean and covariance of $C$, based on observed data. As above, it would be unlikely for the prediction learner to have this true model class. It remains an important open theoretical question how such approximations influence the behavior agent's ability to maximize Information Gain.

\newcommand{\figwfour}{0.24\textwidth}
\newcommand{\figwthree}{0.33\textwidth}
\newcommand{\figwtwo}{0.48\textwidth}

\section{Experimental Setup}

We conducted five experiments, across the three problems described in Section \ref{sec:problems}. The goal of these experiments is to (a) assess the utility of different intrinsic rewards in our testbed with many different target distributions, and (b) to understand how the ability of the underlying prediction learners---introspective or not---impact the results. 

Each component of the learning system is modulated by several hyper-parameters that interact in different ways. The behavior agent (gradient bandit) makes use of a step-size parameter $\alpha$ and the step-size parameter of the average reward estimate $\alpha_r$. For non-introspective learners, each prediction learner makes use of a (shared) step-size parameter $\alpha_p$, with $\alpha_i = \alpha_p$ for all $i$. For introspective learners, the step-size adaption method Autostep uses a meta learning-rate parameter $\kappa$. Finally, many of the intrinsic rewards have their own tunable parameters. For example, \emph{UDE} uses an exponential average of recent errors which requires a smoothing parameter $\beta$. Oudeyer's \emph{Error Derivative} reward makes use of two windows of recent errors determined by scalar parameters $\eta$ and $\tau$. In most cases the key parameters of the prediction learner, behavior agent, and intrinsic reward correspond to different timescales---slower or faster---and so required noticeably different values. 
Because these choices have such a big impact on behavior, as we show, we needed extensive sweeps and analysis to gain insight into the methods. This warranted investigating each result deeply, to communicate a nuanced picture.  

We extensively sweep all the key performance parameters of every learner and reward function, to ensure an accurate characterization of performance. 
Table \ref{params} lists all the parameter settings we tested. In some cases we report results for several parameters to gain more specific insights into the behavior induced by an intrinsic reward. When providing overall results, we report the best performance of the learning system for each intrinsic reward, using the best performing parameters across all parameters tested. The best performing parameters were those that achieved the lowest total RMSE (defined in Equation \ref{eq:MSE_arms}) over the duration of the experiment, averaged over 200 independent runs. All told we tested over 50,000 parameter configurations, 200 times each across our three experiments.

\begin{table}[h!]
\caption{\label{params}The hyper-parameter configurations investigated across all three experiments. There was a total of 50,000 combinations of intrinsic reward function and hyper-parameter setting, with each of these evaluated using 200 independent runs.}
\begin{tcolorbox}[tab2,tabularx={X||X|},title=Hyper-parameters,boxrule=0.5pt]
{\bf Behavior agent}  & {\bf Step-size parameter}        $\alpha\in$\{$2^{-8}$, $2^{-7}$, ... , $2^{-2}$\} \\    
{(Gradient Bandit)} &   {\bf Average Reward rate}              \\
& $\alpha_r\in$\{$10^{-5}$, $10^{-4}$, ... , $10^{-1}$\} \\\hline\hline
{\bf Behavior agent}  & {\bf Step-size parameter}        $\alpha\in$\{$2^{-8}$, $2^{-7}$, ... , $2^{-2}$\} \\    
(Dynamic Thompson Sampling) &   {\bf Initial mean estimate}~$m_a = 100$              \\\hline\hline
{\bf Non-introspective prediction learner} & {\bf Step-size parameter}               \\
(LMS with a constant step-size parameter) &  $\alpha_p\in$\{$2^{-7}$, $2^{-6}$, ... , $2^{-2}$\} , \ \ with $\alpha_i = \alpha_p$              \\\hline\hline
{\bf Introspective prediction learner} & {\bf Meta learning-rate}       \\
(LMS with Autostep) &  $\kappa\in$\{$0$, $0.01$, $0.05$, $0.1$\}     \\
   & {\bf Initial step-size} $\alpha_{0,i} = 1.0$       \\\hline\hline
{\bf Smoothing parameter}  &     \\
(Variance of Prediction, Uncertainty Change, Bayesian Surprise, UDE, Expected Error) & $\beta\in$\{$10^{-6}$, $10^{-5}$, ... , $10^{-1}$\}     \\\hline
{\bf Bayesian Surprise Tolerance}  & $\epsilon\in$\{$10^{-5}$, $10^{-4}$, $10^{-3}$\}      \\\hline
{\bf Error Derivative Window}   & $\eta\in$\{$1$, $5$, $10$, $25$, $100$, $1000$\}  \\  
(all combinations s.t. $\eta>\tau$)  & $\tau\in$\{$1$, $5$, $10$, $25$, $100$\}
\end{tcolorbox}
\end{table}

When reporting results under the best parameters, we jointly tune hyper-parameters for the intrinsic reward and the prediction learners. These hyper-parameters are all part of the agent; the best hyper-parameters reflect the best the agent could do for that intrinsic reward and prediction learner. Even under ideal circumstances, many intrinsic rewards can fail to induce the desired behavior, highlighting issues with the intrinsic rewards or with the use of non-introspective learners. 

Nonetheless, reporting the best parameters does not provide the full picture, and though we attempt to highlight certain key results for other hyper-parameters, we cannot and do not attempt to show the full picture. Ideally, we could slice down further, to provide this nuance. At the extreme, this could consist of showing all possible intrinsic rewards---{\em Error Derivative} with the smallest step-size parameter, {\em Error Derivative} with a the largest step-size parameter, and so on---for each of the many combinations of prediction learner and behavior agent (with different hyper-parameter settings). This is infeasible\footnote{To enable the reader to do this on their own, we have provided a python notebook to explore the full set of data, at \url{http://jair.adaptingbehavior.com} .}, and part of the role of the empirical analysis is to summarize key outcomes. We have provided what we believe are the key slices: different intrinsic rewards (under their ideal circumstances) with two types of prediction learners (non-introspective and introspective). When intrinsic rewards fail under idealized scenarios, this reflects how they might perform across hyper-parameter settings. When intrinsic rewards result in reasonable behavior, we then dig deeper to understand if this was an accident of idealized hyper-parameter tuning, or more generally a characteristic of the intrinsic reward.

We follow the same basic template in the presentation of the results. First we report the behavior of the best configuration for each reward function using non-introspective learners ({\em i.e.,} without Autostep). For a given reward, the behavior is depicted by the probability of selecting each action over time according to the behavior agent's policy. This gives us insight into how each reward drives action selection over time. We then investigate the RMSE over time, plotting both the error of each predictor and the average. Finally, in each experiment we investigate the performance sensitivity of several intrinsic rewards with respect to the tunable parameters. This provides more detailed understanding of how the parameters interact and helps explain when some intrinsic rewards produce unexpected behaviors.        

\section{Experiment One: Drifter-Distractor}

We start with our simplest task: the Drifter-Distractor problem. This problem has 1 constant target, 2 high-variance targets and 1 drifting target (see Figure \ref{fig:cumulants} in Section \ref{sec:problems}). This four-action problem highlights some key features we want out of our learned behaviors. The behavior should not be continually distracted by noisy or unlearnable things (the two high-variance targets). It should be able to quickly learn about simple targets (the constant target), and ultimately focus action selection on targets that result in continual learning progress (the drifting target). We test if such a behavior is learned, with non-introspective and introspective learners, under different intrinsic rewards.  

Let us first be more precise about how the behavior should look in this problem. The behavior should try out all the actions in the beginning. The prediction learner associated with the constant target should quickly reduce its error and the behavior should stop selecting the corresponding action. The prediction learners associated with the high-variance targets will take longer to learn due to the target variance, but eventually should converge to the correct prediction of zero. Once that happens the behavior should stop choosing the actions corresponding to the high-variance targets. Finally, the prediction learner corresponding to the drifting target cannot ever reduce its error to zero: unending learning progress is possible. Therefore the behavior should eventually settle on selecting the action corresponding to the drifting target the majority of the time.

There are a few common degenerate behaviors that are possible in this problem. The first is over-selecting the actions corresponding to the high-variance targets. Every time the behavior takes one of these actions, the corresponding non-introspective prediction learner updates toward a random target and so its predictions can oscillate around the optimum. Over short windows of time, the variance of the drifting target is smaller than the high-variance targets; within that window, the errors generated by the high-variance targets will appear larger. This results in the behavior frequently selecting the high-variance targets, occasionally selecting the drifting target and cycling between the three. Any methods that rely on prediction learners to not chase noise should exhibit this degenerate behavior, such as \emph{Weight Change}. With non-introspective learners, this can only be prevented if the intrinsic reward can somehow between distinguish high-variance and drifting targets. 

The other common degenerate behavior is selecting all actions nearly equally. This strategy does not result in the lowest possible RMSE, but it does result in lower RMSE than other behaviors such as mostly selecting the actions corresponding to the high-variance targets. The uniform strategy emerges because there is no setting of the parameters of the intrinsic reward to force the behavior to follow the ideal strategy described above. 

 \begin{figure}[t]
      \centering
      \includegraphics[width=1.0\linewidth]{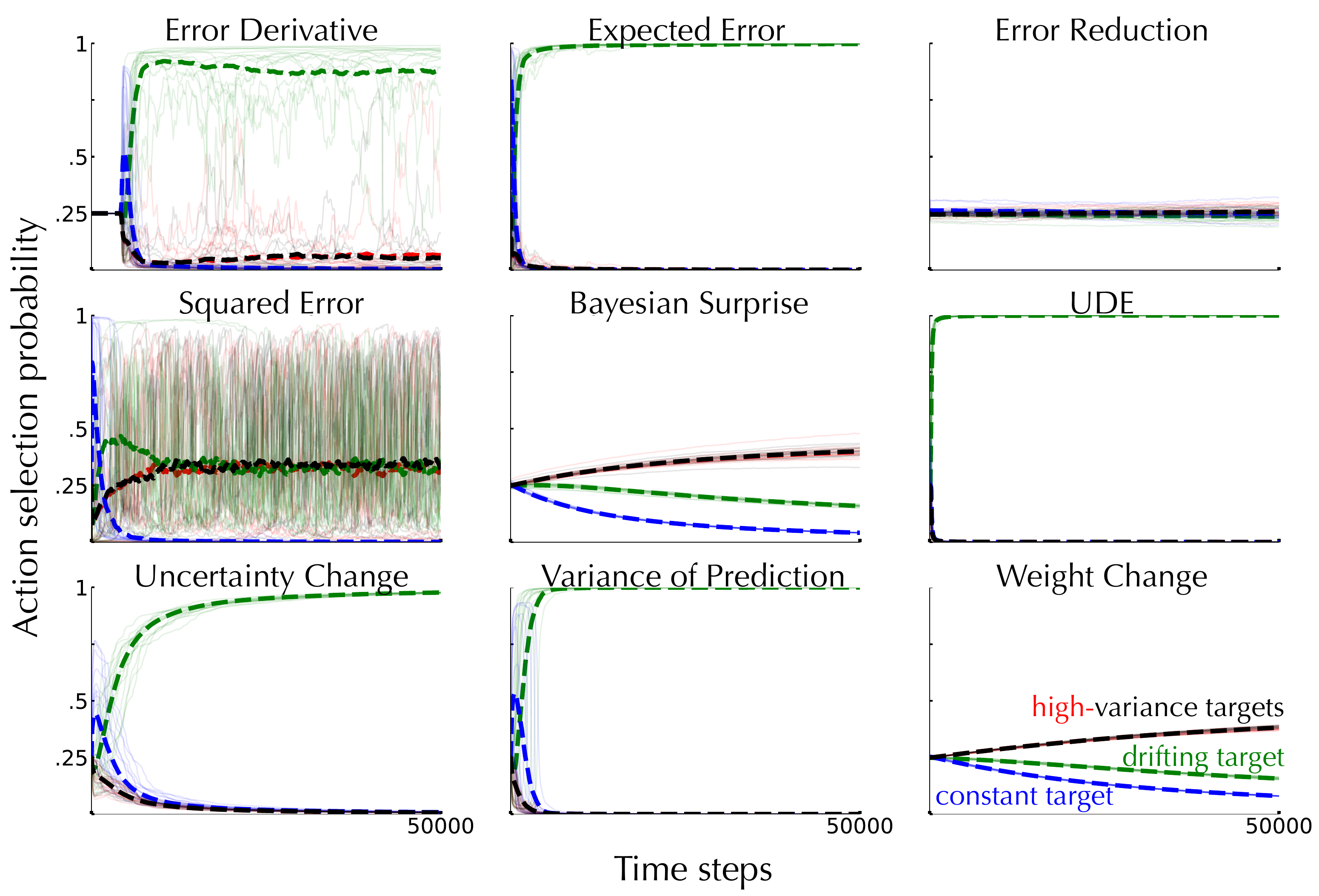}
      \caption{{\textbf{Behavior} in the \textbf{Drifter-Distractor} problem, with \textbf{Non-Introspective Learners}.} 
      Each subplot corresponds to the behavior of the Gradient Bandit with a different intrinsic reward. Each line depicts the action selection probabilities learned by the behavior agent, over 50000 steps. The bold dashed lines show the mean probability of each action, averaged over 200 repetitions of the experiment. The light stroke solid lines show the probabilities computed by the Gradient Bandit for each action on individual runs---we only show a small subset (15 per action) of runs for readability. The \textcolor[rgb]{0,.5,0}{\bf green line} corresponds to the drifting target, the \textcolor[rgb]{0,0,1}{\bf blue line} corresponds to the constant target, and the \textcolor[rgb]{1,0,0}{\bf red} and {\bf black lines} correspond to the high-variance targets. Intrinsic rewards based on variance estimates and averaging errors over time induce sensible action selection. 
      }
      \label{fig:r1}
 \end{figure}
  
\subsection{Results with Non-introspective Learners}

Figure \ref{fig:r1} summarizes the behavior of the Gradient Bandit with several intrinsic reward functions, with non-introspective learners. The bold dash lines reflect the probabilities averaged over 200 runs, while the light stroke solid lines depict probabilities of individual runs. Several rewards induced the ideal behavior described above to varying degrees. Rewards based on simple moving averages of each learner's prediction error, including {\em Expected Error} and {\em UDE}, quickly latch onto the action corresponding to the drifting target. This was possible because the parameter sweep choose a short averaging window, allowing the Gradient Bandit to quickly identify the noisy, high-variance targets---other window lengths caused {\em Expected Error} and {\em UDE} to focus on the high variance targets. Using the variance of each predictors estimate, as in {\em Variance of Prediction} and {\em Uncertainty Reduction}, the behavior also converges to mostly selecting the drift action, after exploring the constant and high-variance actions initially a bit longer. In this case, a parameter corresponding to a long window is used, because the drifting target exhibits higher variance than the high-variance targets over a long enough window of data. Perhaps unsurprisingly the {\em Squared Error} and {\em Error Reduction} produce inappropriate behavior. {\em Bayesian Surprise} and {\em Weight Change} cause the Gradient Bandit to be distracted by the high-variance targets resulting in sub-optimal behavior. The {\em Error Derivative} reward induces behavior that looks reasonable in expectation, albeit there is more variance across runs than exhibited by other intrinsic reward functions. 

%
%

   \begin{figure}[ht]
      \centering
       \begin{tikzpicture}
  \node (img)  {\includegraphics[width=0.95\linewidth]{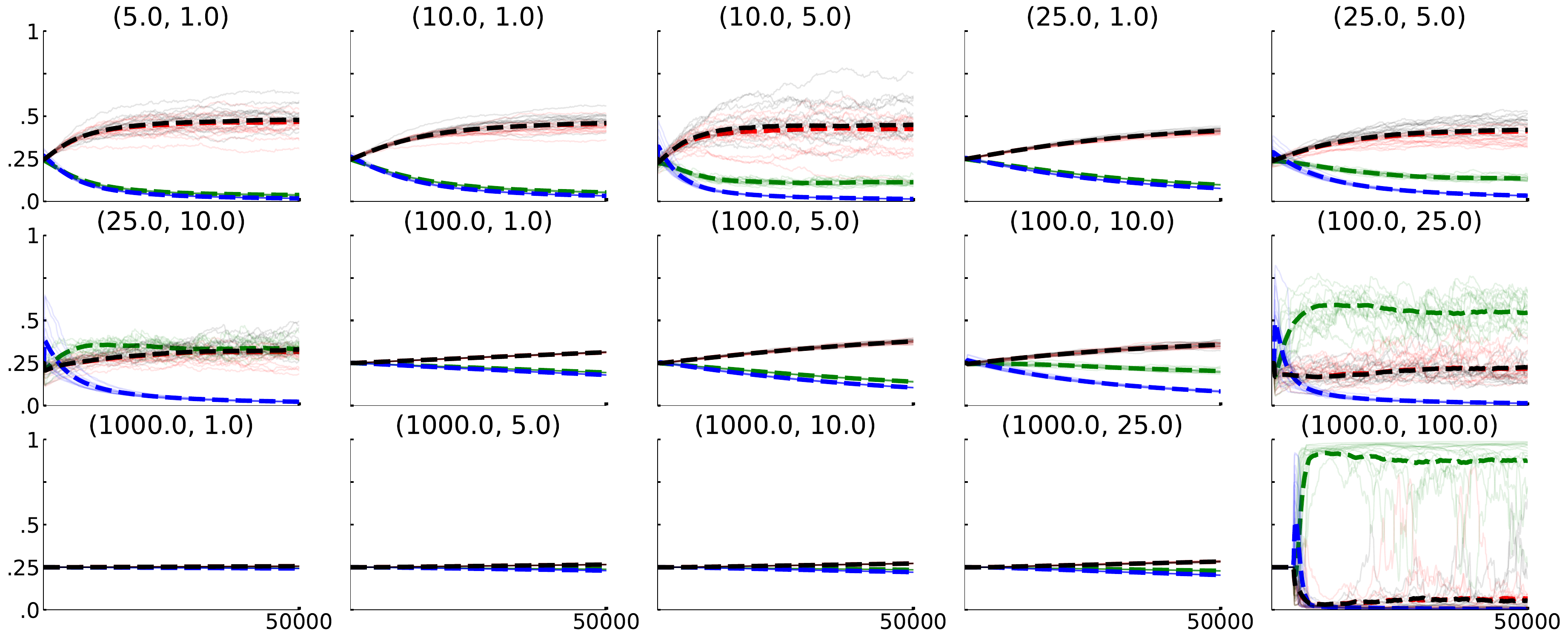}};
  \node[below=of img, node distance=0cm, yshift=1cm,font=\color{black}] {Time steps};
  \node[left=of img, node distance=0cm, rotate=90, anchor=center,yshift=-0.7cm,font=\color{black}] {Action selection probability};
 \end{tikzpicture}
      \caption{{The impact of varying the window length parameters $\eta$ and $\tau$ of {\bf Error Derivative} reward}, in the \textbf{Drifter-Distractor} problem, with \textbf{Non-introspective Learners}. Each subplot depicts the behavior of the Gradient Bandit algorithm with {\em Error Derivative} reward for many combinations of $\eta$, $\tau$ as indicated by the labels. As in Figure \ref{fig:r1}, each subplot shows both the average action selection probability for each action over time, and a small subset of individual runs. A large diversity of behaviors can be induced by changes to the window length parameters. Only one setting induced correct behavior: $\eta = 1000, \tau = 100$. This explains why the initial action selection was uniform in Figure \ref{fig:r1}: the reward is zero until the windows fill, which takes 1000 steps for $\eta = 1000$.} 
      \label{fig:Derv_window}
 \end{figure}

Performance in the Drifter-Distractor problem with non-introspective learners is largely dependent on setting the hyper-parameters of the each reward correctly. To illustrate this sensitivity, consider the {\em Error Derivative} reward, which is parameterized by two scalars $\eta$ and $\tau$. The $\eta$ parameter controls the size of the window used to average recent errors, and $\tau$ controls how much each of the two windows overlap. Figure \ref{fig:Derv_window} shows the behavior of the Gradient Bandit, in terms of action selection probability over time, for every combination of $\eta$ and $\tau$. For each pair of ($\eta$, $\tau$) we selected all the other hyper-parameters in the learning system to minimize the total RMSE; each subplot of the figure represents the best performance possible for a given ($\eta$, $\tau$) pair according to RMSE. Across these combinations, we see the full gamut of behaviors. Only one setting out of twelve exhibited the described good behavior; most were uniform or focused on the distractor targets. 

 
  \begin{figure}[ht]
      \centering
       \begin{tikzpicture}
  \node (img)  {\includegraphics[width=0.95\linewidth]{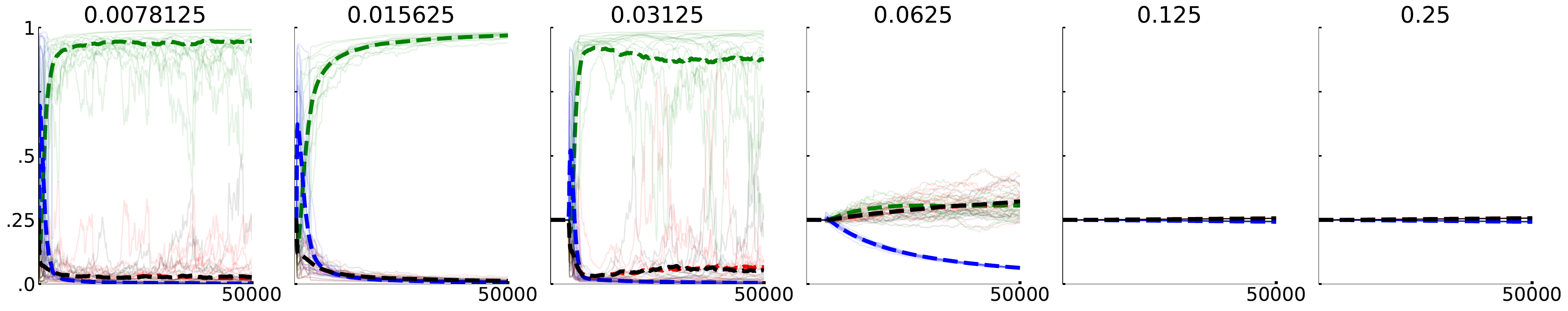}};
  \node[below=of img, node distance=0cm, yshift=1cm,font=\color{black}] {Time steps};
  \node[left=of img, node distance=0cm, rotate=90, anchor=center,yshift=-0.7cm,font=\color{black}] {Action probability};
 \end{tikzpicture}
      \caption{{The impact of varying the LMS step-size parameter $\alpha_p$ with the {\bf Error Derivative} reward}, in the \textbf{Drifter-Distractor} problem, with \textbf{Non-introspective Learners}. Each subplot depicts the behavior of the Gradient Bandit algorithm with {\em Error Derivative} reward for for different values of $\alpha_p$ as indicated by the labels. Large $\alpha_p$---faster target tracking---induces a uniform behavior, and smaller $\alpha_p$ produce more sensible action selection but RMSE is higher because predictions are learned slowly. The third subplot, corresponding to $\alpha_p = 0.03125$, achieved the lowest total RMSE, because it allowed for somewhat faster learning for the predictions, but was still slow enough for the behavior to estimate learning.}
      \label{fig:Derv_alpha}
 \end{figure}
   
The hyper-parameters of the other components of the learning system also interact with the reward function. Figure \ref{fig:Derv_alpha} shows the best behavior---in terms of RMSE---of the Gradient Bandit for different values of the LMS predictor step-size parameter $\alpha_p$. As the predictors learn faster, the \emph{Error Derivative} reward induces nearly uniform action selection. If we slow the prediction learners updates with a smaller step-size parameter value, then the behavior strongly favors the drifting action. This makes sense because with a small $\alpha_p$, the intrinsic reward for the high-variance targets becomes smaller and much bigger for the drifting target because the step-size parameter value is not large enough to track quickly. Though the action selection by the behavior is correct, this is not what we want from the learning system: we want the prediction learners to learn quickly, rather than artificially slowly so that the behavior can more easily track what they know. In fact, with small step-size parameter values, the RMSE is much worse than we can get with the introspective learners, where it is much easier to estimate learning progress and prediction learners can learn more aggressively. 

  \begin{figure}[h!]
      \centering
      \includegraphics[width=1.0\linewidth]{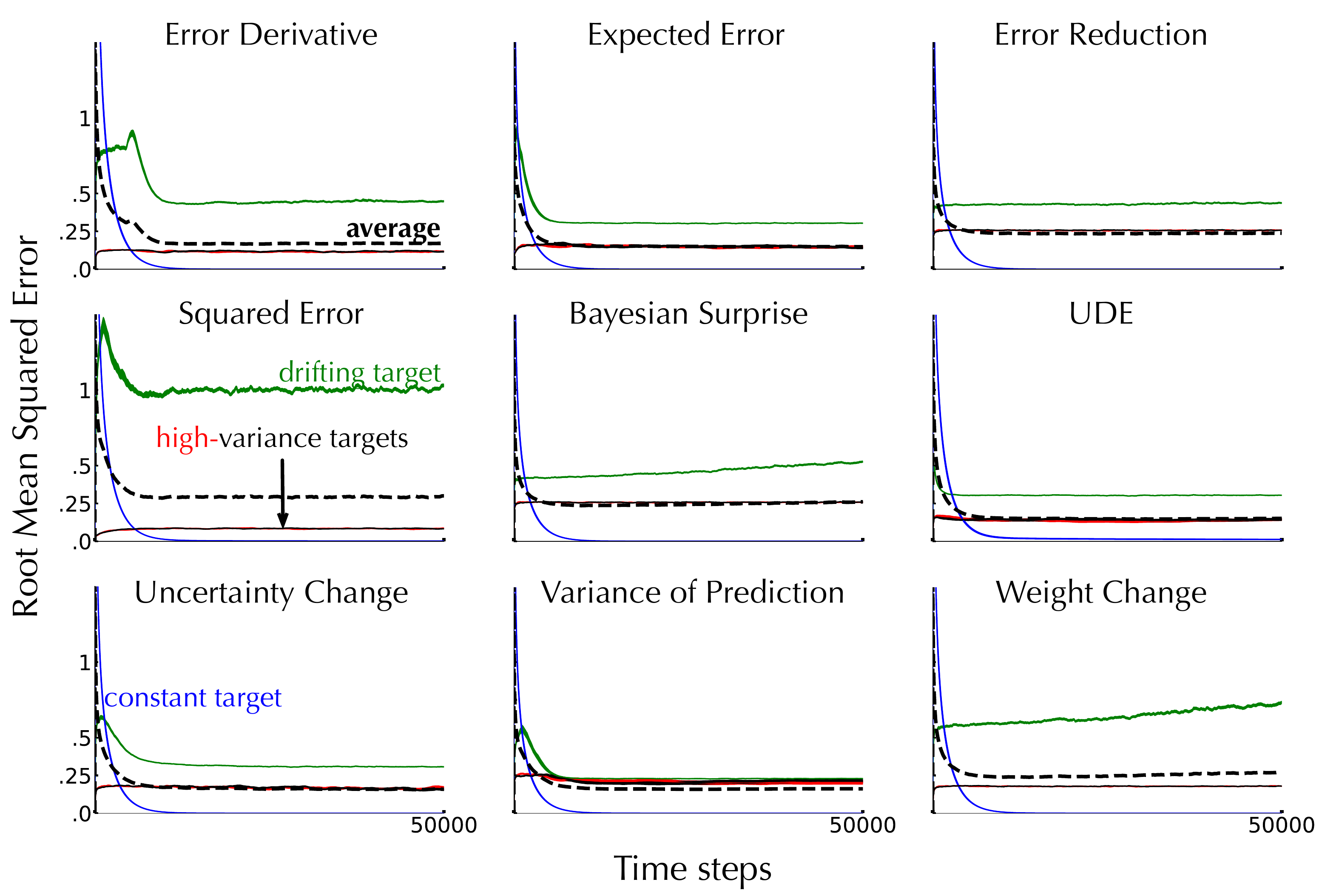}
      \caption{{\textbf{RMSE} over time corresponding to each intrinsic reward function in the \textbf{Drifter-Distractor} problem with \textbf{Non-introspective Learners}.} Each subplot corresponds to a different reward as labelled. The line colors correspond exactly as in the previous plots: \textcolor[rgb]{0,.5,0}{\bf green drifting}, {\bf black} and \textcolor[rgb]{1,0,0}{\bf red} {\bf high}-\textcolor[rgb]{1,0,0}{\bf variance}, and \textcolor[rgb]{0,0,1}{\bf blue constant}. Each line is the exponentially weighted moving average of the LMS predictor's RMSE. The RMSE is computed with an exponential average, with weighting 0.001. The final results are averaged over 200 independent runs (standard error bars are plotted). The heavy stroke black dashed line reports average of the other four. Although many rewards induce similar action selection strategies, they can produce different RMSE curves. 
      }
      \label{fig:rmse_exp1}
 \end{figure}
 
Finally, let us investigate the error over time for each intrinsic reward. Figure \ref{fig:rmse_exp1} shows the exponential average of the RMSE over time for each reward function. We choose an exponential average to smooth the results (with a decay constant of 0.999). We plot both the error of each target, and the error average across targets. All rewards except {\em UDE} result in perfect prediction of the constant target; even {\em UDE} has near-zero error, indicating only minor under-selection of the constant action. Rewards that induce nearly uniform action selection generate larger prediction error in aggregate ({\em Error Reduction} and {\em Surprise}). Reward functions that do not induce a strong preference for the drifting target exhibit high or growing error ({\em Weight Change}). Rewards that induce strong preference for the high-variance targets do achieve better error on those predictions at the cost of accuracy in predicting the drifting targets ({\em Squared Error}). 
Achieving the lowest overall error requires first selecting the actions for the constant and high-variance targets at first, and then focusing on the drift target (i.e., {\em UDE}, {\em Uncertainty Change}, and {\em Variance of Prediction}). 


\subsection{Results with Introspective Learners}


In this section we analyze the impact of different intrinsic rewards with introspective learners. We use LMS learners with Autostep, a step-size adaption method, to obtain introspective prediction learners. 
First let us recall how the step-size parameter for each LMS learner might change over time (see Figure \ref{fig:alphas} in Section \ref{isl} for reference), based on the errors generated by each of our three target types. 
The high-variance targets are impossible to predict---even if the mean is stable---so the LMS learner will experience positive and negative errors. The Autostep algorithm will reduce the step-size parameter corresponding to these targets, allowing each LMS learner to mitigate the variance and converge to the correct prediction of zero. The constant target on the other hand is easy to predict. Autostep will keep the step-size parameter large because the errors will be of the same sign. However, the error on the constant error can easily be reduced to zero with repeated sampling. 
Once the prediction error is zero Autostep will modify the step-size parameter no further. The drifting target has noise, like the high-variance targets, but the mean is not centered at zero, and it exhibits temporal structure. Consequently, the Autostep algorithm will keep the step-size parameter value high for the duration of the experiment. It is not hard to see that introspective learners should efficiently reduce error across all the targets, at least compared with a global, constant step-size parameter value. More subtly, an intrinsic reward that takes into account the dynamic values of the step-size parameter could exploit this additional information to adapt behavior to reduce error even faster.

\begin{figure}[h]
      \centering
      \includegraphics[width=1.0\linewidth]{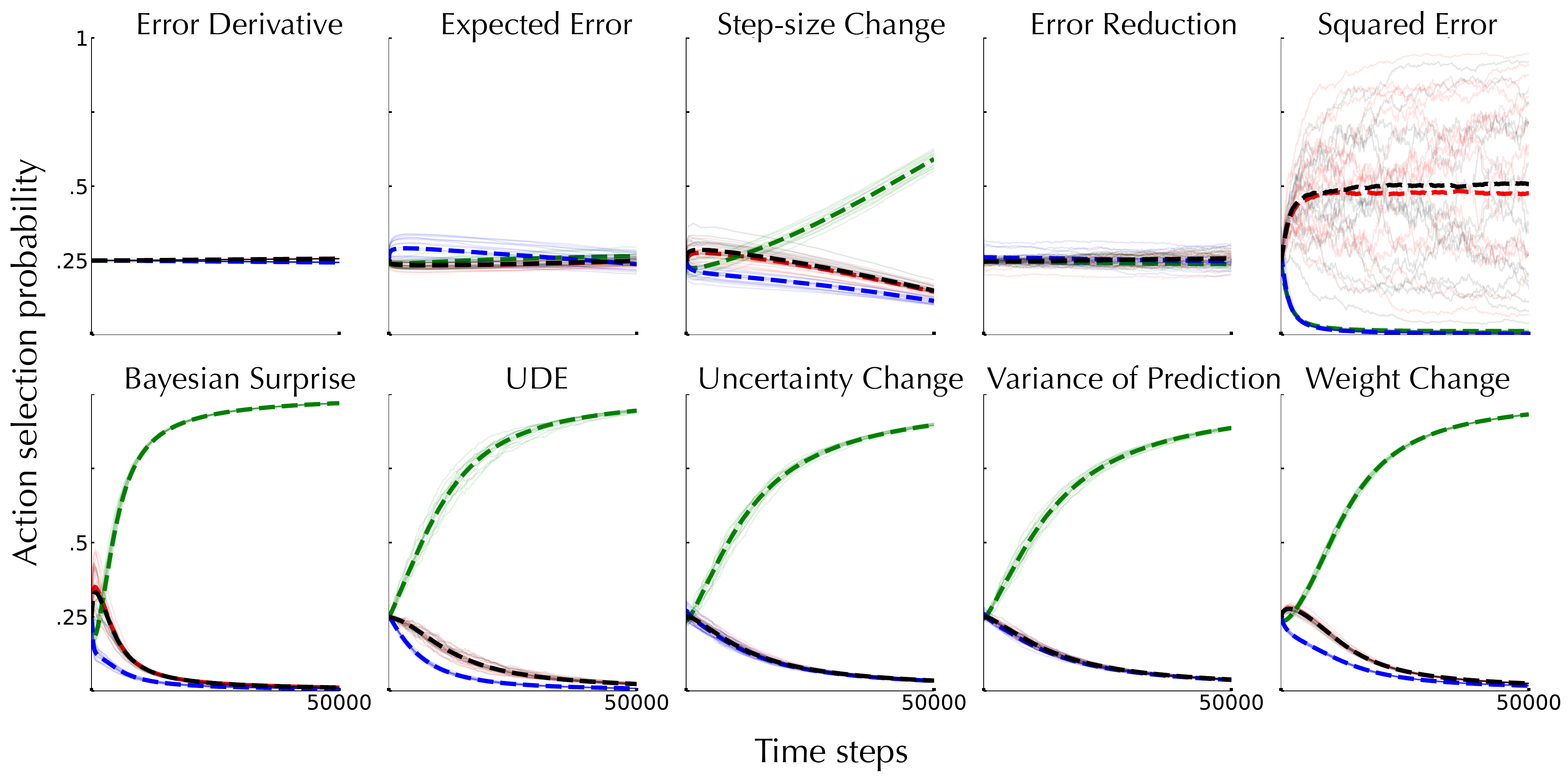}
 \caption{{\textbf{Behavior} in the \textbf{Drifter-Distractor} problem with {\bf Introspective Learners}.} 
      Each subplot corresponds to the behavior of the Gradient Bandit with a different intrinsic reward. Each LMS learner uses the Autostep algorithm to adapt the step-size parameter over time. The line coloring, labelling, and semantics mirror Figure \ref{fig:r1}. With Autostep, {\em Weight Change} induces sensible action selection. {\em Error Derivative} and {\em Expected Error} rewards, on the other hand, induce inappropriate action selection.     
      }
      \label{fig:r2}

\end{figure}

The setup of our second experiment was identical to the first except that each LMS learner maintained its own step-size parameter $\alpha_{t,i}$ updated via Autostep. We also include an intrinsic reward based on the change in the step-size parameter to assess the utility of rewarding action choices that caused changes in the step-size parameter values. This reward only makes sense if the step-size parameter can change over time, and thus was not included in the previous experiment. 

The results of our second experiment are summarized in Figure \ref{fig:r2}. As before we plot the action selection probabilities to summarize the behavior. Weight change reward now induces reasonable action selection. The step-size parameters for the high-variance targets decay to a relatively small values causing the weight change to reduce---those actions become less and less rewarding. Autostep keeps the step-size parameter value relatively high for the drifting target, on the other hand, and the change in weights remains relatively high. Finally, even though the step-size parameter does not decay to zero for the constant target, the prediction error for the constant target does go to zero. Consequently, the magnitude of the update also goes to zero, meaning the weight change goes to zero and preference for the constant action diminishes over time. {\em Bayesian Surprise} induces similar behavior as {\em Weight Change} as suggested by our analysis in Section \ref{sec_bayes}. The variance-based rewards and {\em UDE} induce the same overall action preferences as without Autostep.



Across the board there is an improvement in RMSE reduction as shown in Figure \ref{fig:rmse_exp2}. The RMSE is about half of that for the non-introspective learners. The differences in RMSE between the intrinsic rewards appear more minor, but the differences are meaningful. The total RMSE is well correlated with our definition of reasonable behavior in this domain---reward functions that result in lower error exhibit the expected action preferences over time. To see larger differences, though, we need more actions. This first experiment was primarily designed to investigate qualitative behavior; the final experiment uses more actions and provides a better insight into quantitative differences. 

  \begin{figure}[ht]
      \centering
      \includegraphics[width=1.0\linewidth]{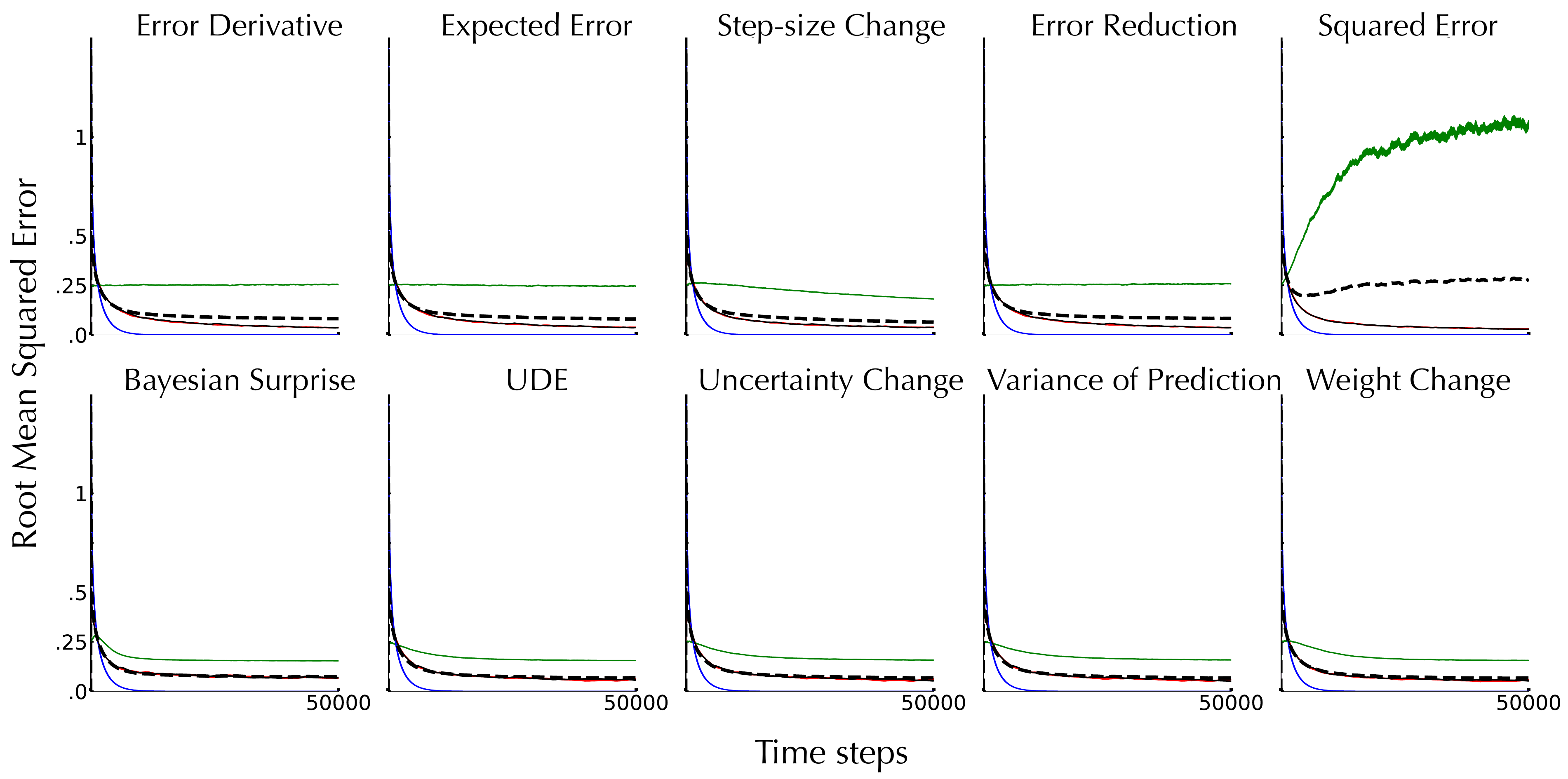}
      \caption{{\textbf{RMSE} over time corresponding to each intrinsic reward function in the \textbf{Drifter-Distractor} problem with {\bf Introspective Learners}.}  In this experiment, reward functions that induce similar action preferences produce similar RMSE reduction over profiles. Using {\em Weight Change} reward produces the lowest RMSE (0.108), however both {\em UDE} (0.109) and {\em Uncertainty Change} (0.110) result in similar performance. {\em Squared Error} results in the worst performance overall (0.292), and rewards that induce uniform action selection like {\em Error Derivative} result in larger error (0.124) compared with {\em Weight Change}.       }
      \label{fig:rmse_exp2}
 \end{figure}
 
For non-introspective learners, we observed that careful tuning of hyper-parameters allowed for the correct behavior for certain intrinsic rewards, by slowing prediction learning. This was the case for the {\em Error Derivative}, where in Figure \ref{fig:Derv_alpha} we observed that if the predictors learned too quickly, the drifting target did not produce the highest {\em Error Derivative}. For introspective learners, prediction learning cannot be slowed: they increase learning when learning is possible. We might expect {\em Error Derivative} to therefore perform poorly, and be unable to find an appropriate hyper-parameter setting. We find this is the case: {\em Error Derivative} with Autostep does not induce the action selection preferences we expect---it causes nearly uniform action behavior---and no setting of the window parameters resulted in appropriate action preferences (Figure \ref{fig:Derv_window2}). 

Overall, the preference for the drifting action is less pronounced with introspective learners, as seen in Figure \ref{fig:r2}. Instead, the behavior selects the high-variance targets for longer. This is because step-size adaption is a meta (or second order) learning process, and so a non-trivial amount of data is required to recognize that learning is oscillating. 

   \begin{figure}[ht!]
      \centering
       \begin{tikzpicture}
  \node (img)  {\includegraphics[width=0.95\linewidth]{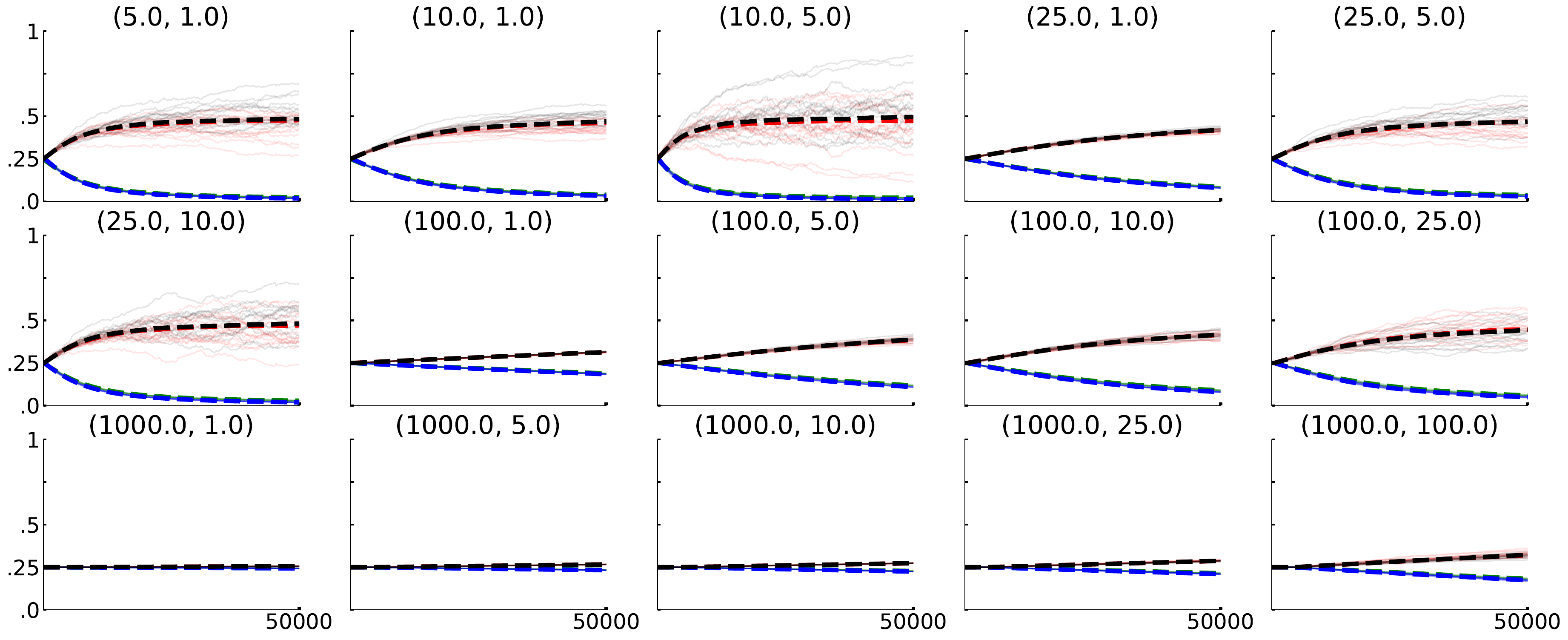}};
  \node[below=of img, node distance=0cm, yshift=1cm,font=\color{black}] {Time steps};
  \node[left=of img, node distance=0cm, rotate=90, anchor=center,yshift=-0.7cm,font=\color{black}] {Action selection probability};
 \end{tikzpicture}     
      \caption{{The impact of varying the window length parameters $\eta$ and $\tau$ of {\bf Error Derivative} reward in the \textbf{Drifter-Distractor} problem with \textbf{Introspective Learners}}.}
      \label{fig:Derv_window2}
 \end{figure}
 
One might therefore wonder if rewards like the {\em Weight Change} reward simply hide the hyper-parameter tuning issue inside the step-size adaption algorithm. This seems not to be the case: the parameters of Autostep are straightforward to tune, and the behavior is largely insensitive to these choices as shown in Figure \ref{fig:wc}. Small meta learning-rate parameter values slow learning but do not prevent preference for the drifting action. 
The results of our first experiment highlight the utility of both simple intrinsic rewards---one's without hyper-parameters---and introspective learners in multi-prediction learning systems.  
   \begin{figure}[ht!]
      \centering
       \begin{tikzpicture}
  \node (img)  {\includegraphics[width=0.5\linewidth]{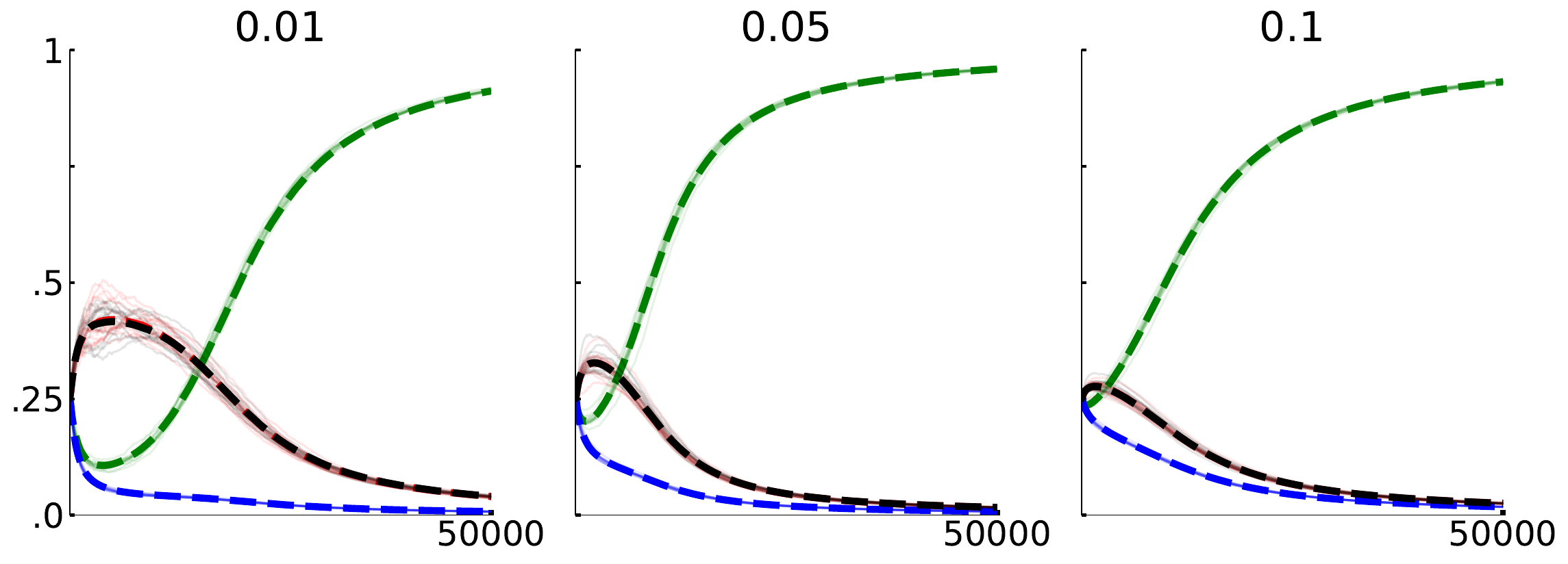}};
  \node[below=of img, node distance=0cm, yshift=1cm,font=\color{black}] {Time steps};
  \node[left=of img, node distance=0cm, rotate=90, anchor=center,yshift=-0.7cm,font=\color{black}] {Action probability};
 \end{tikzpicture}           
        \caption{\label{fig:wc} Action selection probabilities for the Gradient Bandit with \textbf{Weight-Change} reward under different meta learning-rate parameter values $\kappa$, in the \textbf{Drifter-Distractor} problem with \textbf{Introspective Learners}.} 
\end{figure}

One final point of note is the surprising difference between {\em UDE} and {\em Expected Error}. In the previous experiment, with non-introspective learners, they performed similarly.
In this experiment, with introspective learners, {\em Expected Error} results in uniform action selection whereas {\em UDE} provides the correct behavior. This is surprising, as {\em UDE} corresponds to {\em Expected Error} divided by the long-run sample standard deviation of the target. If we look more closely at the behavior induced by {\em Expected Error} with different smoothing parameters $\beta$, in Figure \ref{fig:trace_var}, it becomes more clear why this is the case. A small $\beta$ in this problem results in early errors dominating the moving average; consequently, the constant action is preferred, as it generates high error at first. A larger $\beta$ is needed to avoid this issue, but this unfortunately causes poor estimates of the true expected error for the high-variance targets (which should be zero). In fact, it makes the errors for those target appear higher. Consequently, for the four smaller $\beta$, the constant target is preferred and for the two large, the high-variance targets are preferred; there is no $\beta$ amongst our set that lets the behavior focus on the drifting target. 

{\em UDE}, on the other hand, has a way to overcome this: the long-run variance estimate makes the drift target appear better. The variance of the drift target appears small in the beginning of learning, and it takes many steps to start to recognize that it is actually high-variance. In contrast, the variance estimate for the high-variance targets are learned quickly, and the variance for the constant target looks higher initially due to consistent decrease in the error. This behavior is perhaps a bit accidental, and again highlights the complex interactions between all these hyper-parameters. This only further motivates the utility of intrinsic rewards with no hyper-parameters, that rely on introspective prediction learners.

   \begin{figure}[ht!]
      \centering
       \begin{tikzpicture}
  \node (img)  {\includegraphics[width=0.95\linewidth]{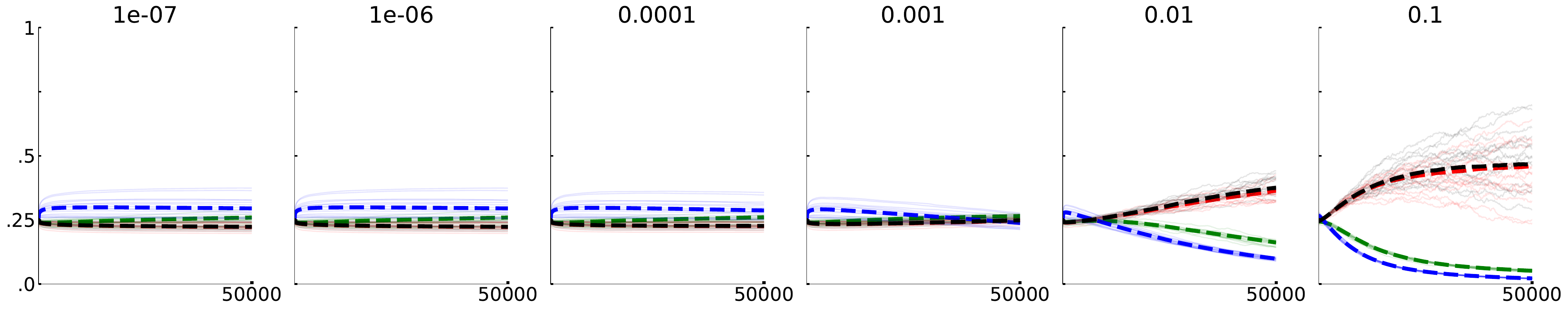}};
  \node[below=of img, node distance=0cm, yshift=1cm,font=\color{black}] {Time steps};
  \node[left=of img, node distance=0cm, rotate=90, anchor=center,yshift=-0.7cm,font=\color{black}] {Action probability};
 \end{tikzpicture}  
      \caption{{The impact of varying the smoothing parameter $\beta$ of {\bf Expected Error} reward in the \textbf{Drifter-Distractor} problem with \textbf{Introspective Learners}}.}
      \label{fig:trace_var}
 \end{figure}

\subsection{Results with Another Bandit Algorithm}

A natural question is if the above results are specific to the Gradient Bandit behavior learner. 
To verify that our conclusions were not somehow overfit to the Gradient Bandit algorithm, we also repeated Experiment One with a Dynamic Thompson Sampling (DTS) algorithm, described in Section \ref{sec_banditalgs}. This bandit algorithm is representative of a different class of algorithms used in online learning: DTS estimates action-values instead of action-preferences and uses optimism (Thompson sampling) to increase exploration.

\begin{figure}[h!]
      \centering
      \includegraphics[width=0.85\linewidth]{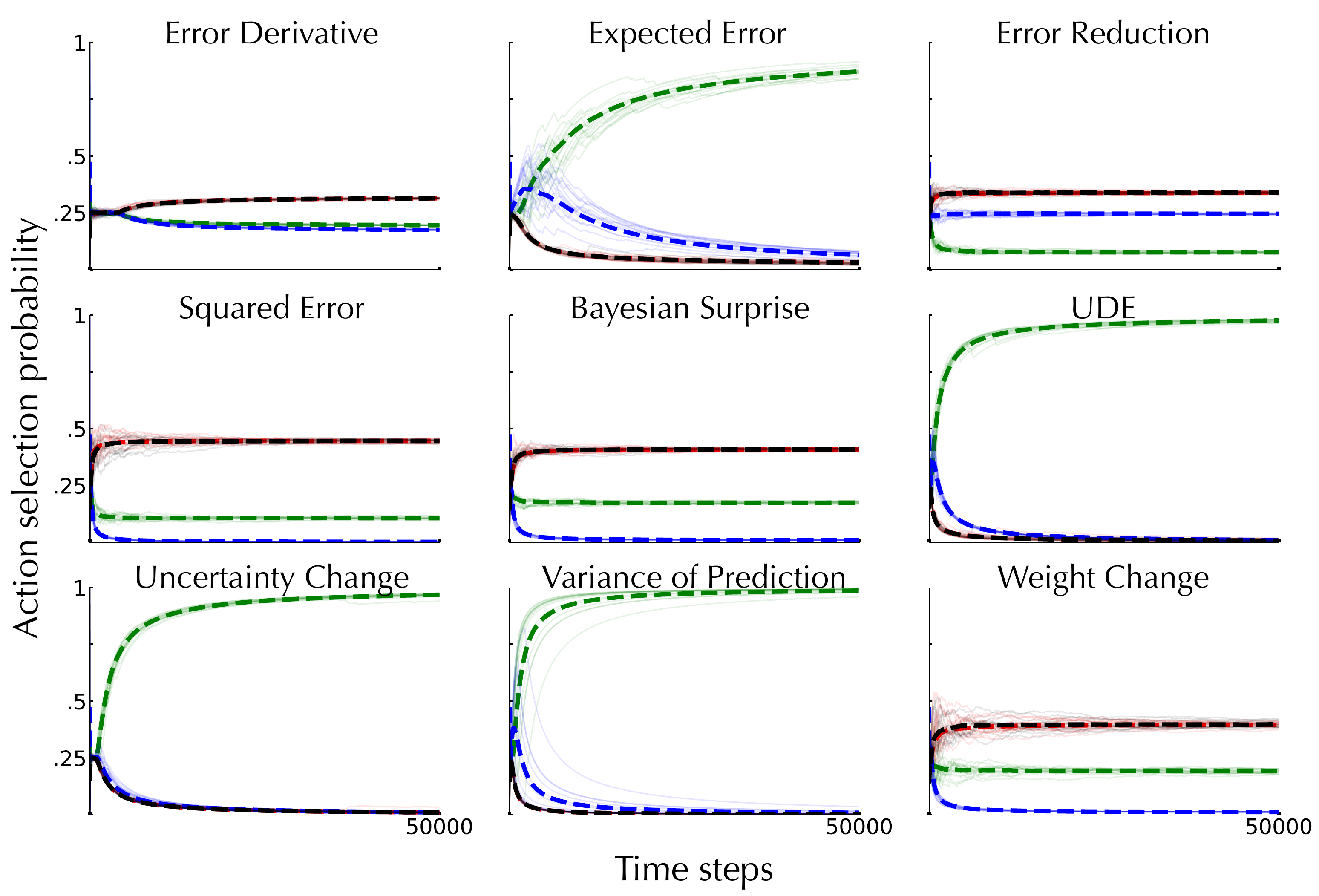}
 \caption{\textbf{Behavior} in the \textbf{Drifter-Distractor} problem, with \textbf{Non-Introspective Learners} where the  behavior is learned using Dynamic Thompson Sampling.
      }
      \label{fig:dts_1}
\end{figure}
\begin{figure}[h!]
      \centering
      \includegraphics[width=1.0\linewidth]{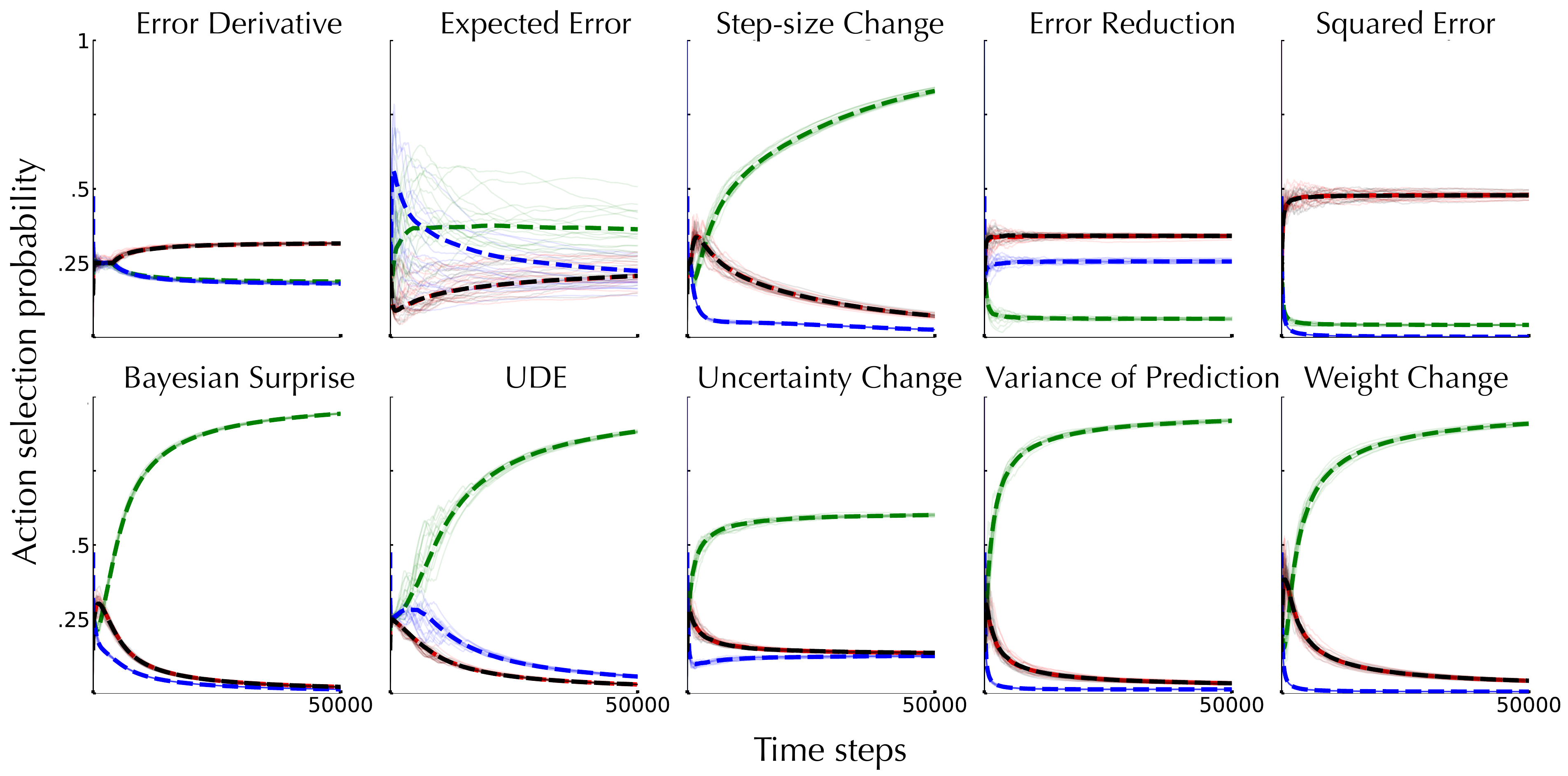}
 \caption{\textbf{Behavior} in the \textbf{Drifter-Distractor} problem, with \textbf{Introspective Learners} where the behavior is learned using Dynamic Thompson Sampling.
      }
      \label{fig:dts_2}
\end{figure}

The qualitative conclusions are similar, shown in Figures \ref{fig:dts_1} and \ref{fig:dts_2}. The primary difference is that the DTS algorithm is less likely to settle on uniform random behavior; rather it is more likely to exhibit a preference. This could be because the algorithm is inherently designed to identify the best action, under drift, whereas the Gradient Bandit samples according to preferences. The Gradient Bandit algorithm is less greedy: if two actions have similar action preferences, even if one is clearly higher than the other, then the Gradient Bandit will spend time selecting both actions. Deterministic methods select the action with the highest value. Despite this difference, we observed the same basic behaviors and the same qualitative differences amongst different intrinsic rewards. 
~The only noteworthy difference is the behavior with {\em Error Derivative} and non-introspective learners: with DTS this reward no longer induces the expected action selection. As we saw, {\em Error Derivative} only worked for a narrow range of its hyper-parameters in Experiment One, and so it is not surprising that it was not a stable result.




\section{Experiment Two: Switched Drifter-Distractor Problem}
Our second experiment is similar to the first except we introduce an unexpected change in the target distributions to tax the reward function's ability to help keep track of the relevant actions. Our first experiment reveal that several intrinsic rewards could help the Gradient Bandit algorithm ignore unhelpful actions and focus on the one corresponding to a drifting prediction target. The intrinsic rewards that were most helpful are based on moving estimates of the error or the variance of the prediction itself. 

To further evaluate these rewards we introduce a simple unpredictable change in the targets. For the first 50,000 steps the task is the same as Experiment One, then the targets suddenly switch according to Table \ref{tab:switch}. The ideal behavior before the {\em switch} should be the same as Experiment One: choose the constant and high-variance actions until their error is reduced, then focus on the drift action. After the switch, two of the targets drift. So the ideal behavior should focus on those two actions equally after some initial transient period due to the change.

This task is partially observable by design. The idea is to simulate a situation where the learning system encounters an unexpected change. The question is how does the sudden change interact with each reward function's internal estimates? Can we find a setting of the smoothing parameter and window lengths the help the intrinsic reward identify the appropriate actions?

\begin{figure}[h!]
      \centering
      \includegraphics[width=0.9\linewidth]{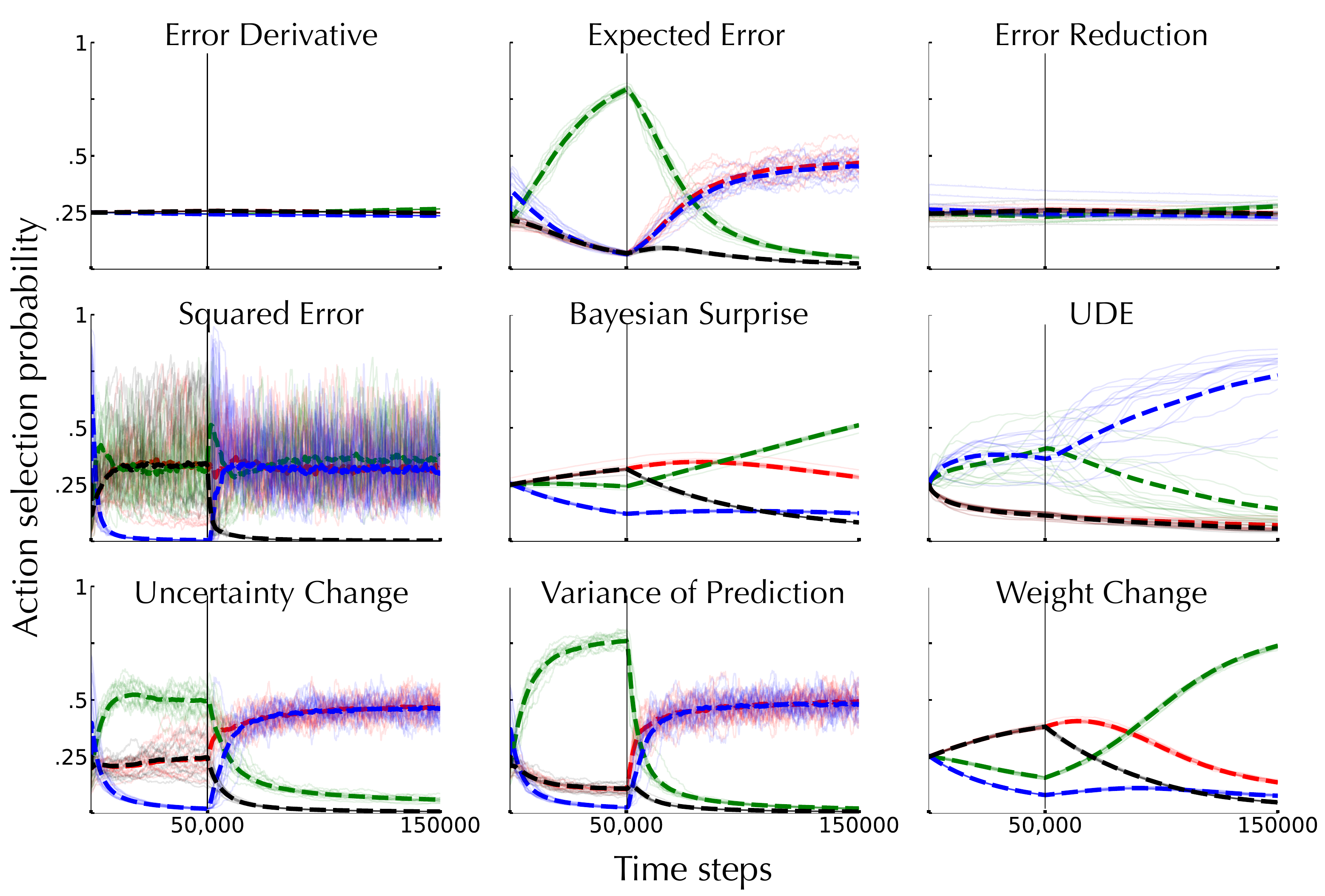}
      \caption{{\textbf{Behavior} in the \textbf{Switched Drifter-Distractor} problem with \textbf{Non-introspective Learners}.} 
      Each subplot corresponds to the behavior of the Gradient Bandit with a different intrinsic reward. The black vertical line in each subplot makes the time of the switch. 
      }
      \label{fig:r3}
 \end{figure}
 
     \begin{figure}[h!]
      \centering
       \begin{tikzpicture}
  \node (img)  {\includegraphics[width=0.9\linewidth]{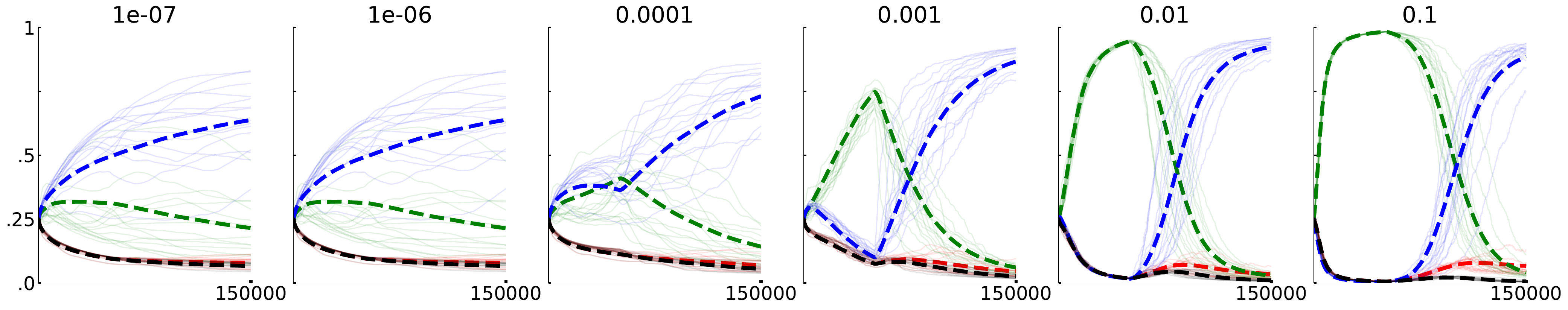}};
  \node[below=of img, node distance=0cm, yshift=1cm,font=\color{black}] {Time steps};
  \node[left=of img, node distance=0cm, rotate=90, anchor=center,yshift=-0.7cm,font=\color{black}] {Action probability};
 \end{tikzpicture}  
      \caption{{The impact of varying the smoothing parameter $\beta$ with the {\bf UDE} reward in the \textbf{Switched Drifter-Distractor} problem with \textbf{Non-introspective Learners}}.  Large $\beta$---more emphasis on recent errors---induces the correct behavior in phase one, but not in phase two.}
      \label{fig:ude_var}
 \end{figure}
 
\subsection{Results with Non-introspective learners}
Figure \ref{fig:r3} summarizes the behavior of the Gradient Bandit with several intrinsic reward functions. In the first phase of the experiment, the behavior is fairly similar to Experiment One. \emph{Weight Change}, \emph{Squared Error}, \emph{Expected Error}, \emph{Bayesian Surprise} and the Variance-based rewards perform almost the same, though there is reduced selection of the drift action. Notably both {\em Error Derivative} and {\em UDE} exhibit substantially different behavior. \emph{Error Derivative} induces uniform action preferences for the entire duration of the experiment. \emph{UDE} has trouble inducing a strong preference between the drift and constant actions in phase one. In the second phase \emph{UDE} correctly focuses action selection on only one of the two drift actions. 
 
 Let us take a closer look first at {\em UDE} to get a better sense of why its behavior is so different in Experiment Two. Figure \ref{fig:ude_var} illustrates how the behavior changes as of function of the smoothing parameter of \emph{UDE}. As in Experiment One, a large value of $\beta$ induces strong preference for the drift action in phase one, however, in phase two mainly one of the drift actions is selected resulting in higher RMSE. The best RMSE is achieved with smaller $\beta$ that (a) over-selects the constant action in phase one, (b) over-selects the high-variance target in phase two, and (c) under-selects the both drift actions in phase two. Figure \ref{fig:ude_alpha_demon} illustrates how the behavior changes as of function of the step-size parameter $\alpha_p$. We see the same phenomenon here with {\em UDE} that we saw with {\em Error Derivative} in Experiment One. If $\alpha_p$ is small, then the prediction of the drifting target is less accurate and thus the behavior induced by {\em UDE} favors the drift action, but results in much higher overall error.

   \begin{figure}[t]
      \centering
       \begin{tikzpicture}
  \node (img)  {\includegraphics[width=0.9\linewidth]{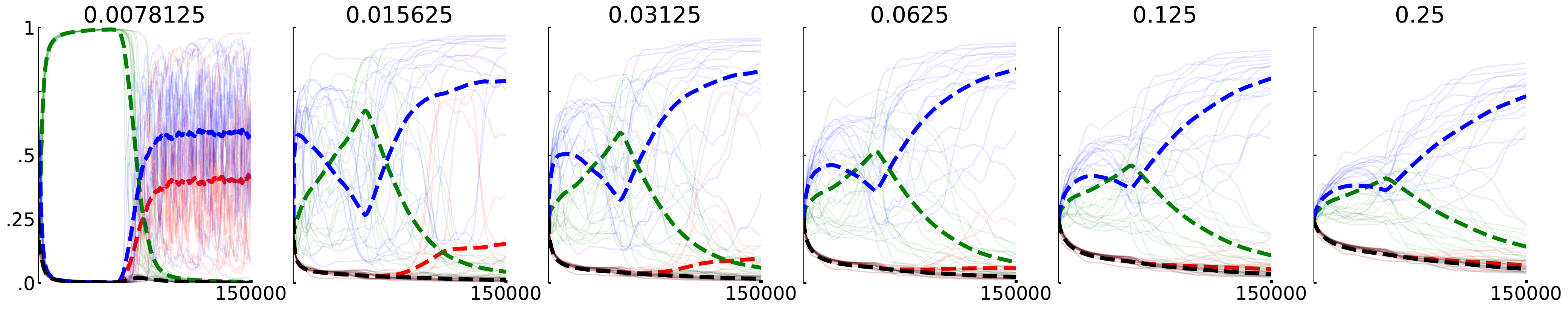}};
  \node[below=of img, node distance=0cm, yshift=1cm,font=\color{black}] {Time steps};
  \node[left=of img, node distance=0cm, rotate=90, anchor=center,yshift=-0.7cm,font=\color{black}] {Action probability};
 \end{tikzpicture}  
      \caption{{The impact of varying the step-size parameter $\alpha_p$ of LMS with the {\bf UDE} reward} in the \textbf{Switched Drifter-Distractor} problem with \textbf{Non-introspective Learners}.  Interestingly very small $\alpha_p$---slow prediction learning---induces the action preferences closer to what we expect. Unfortunately when $\alpha_p = 0.0078125$ the RMSE is 1.376, compared to a RMSE of 0.313 when  $\alpha_p = 0.25$.}
      \label{fig:ude_alpha_demon}
 \end{figure}

 \begin{figure}[t]
      \centering
       \begin{tikzpicture}
  \node (img)  {\includegraphics[width=0.9\linewidth]{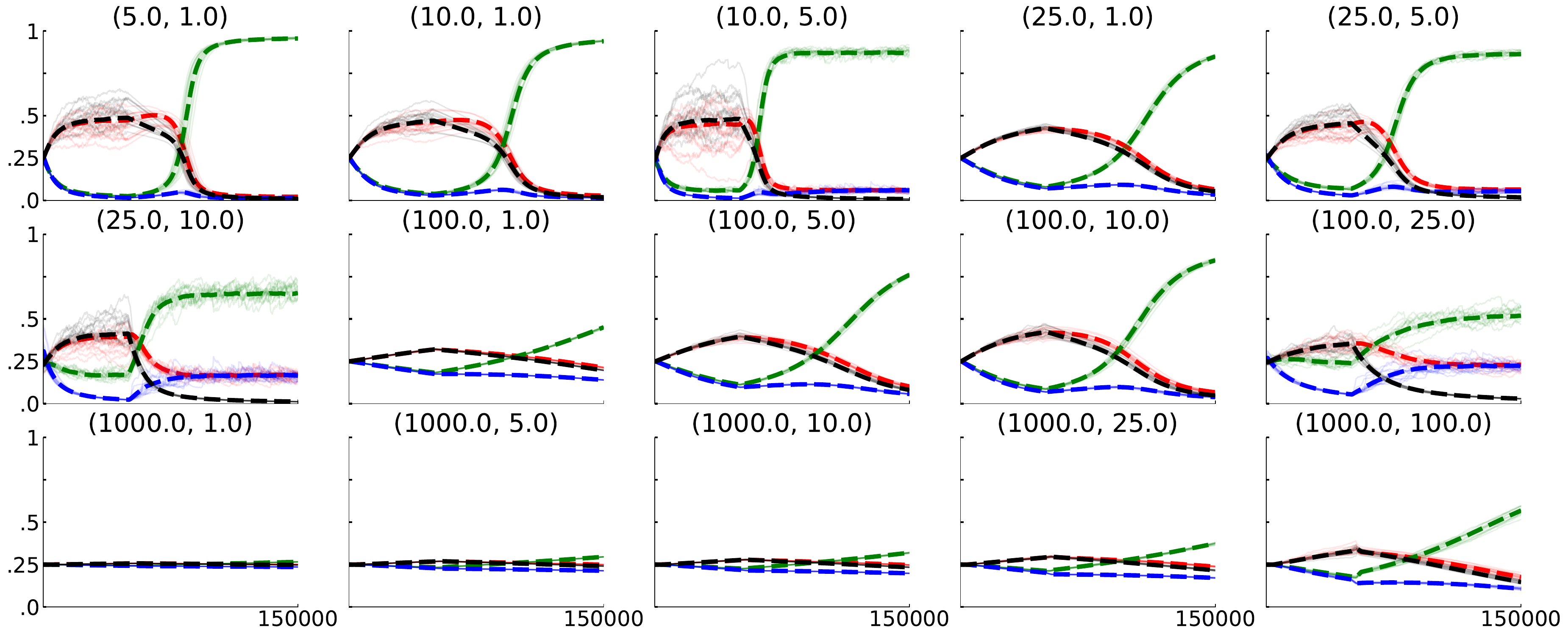}};
  \node[below=of img, node distance=0cm, yshift=1cm,font=\color{black}] {Time steps};
  \node[left=of img, node distance=0cm, rotate=90, anchor=center,yshift=-0.7cm,font=\color{black}] {Action selection probability};
 \end{tikzpicture}  
      \caption{{The impact of varying the window length parameters $\eta$ and $\tau$ of {\bf Error Derivative} reward} in the \textbf{Switched Drifter-Distractor} problem with \textbf{Non-introspective Learners}. }
      \label{fig:exp3_derv_window}
 \end{figure}

   \begin{figure}[h!]
      \centering
       \begin{tikzpicture}
  \node (img)  {\includegraphics[width=0.9\linewidth]{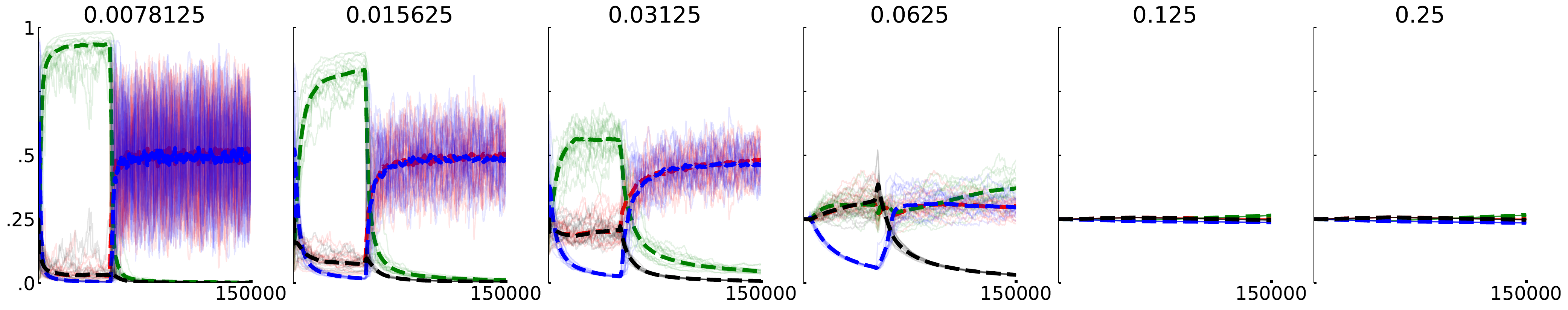}};
  \node[below=of img, node distance=0cm, yshift=1cm,font=\color{black}] {Time steps};
  \node[left=of img, node distance=0cm, rotate=90, anchor=center,yshift=-0.7cm,font=\color{black}] {Action probability};
 \end{tikzpicture}  
      \caption{{The impact of the step-size $\alpha_p$ with the {\bf Error Derivative} reward} in the \textbf{Switched Drifter-Distractor} problem with \textbf{Non-introspective Learners}. When $\alpha_p = 0.0078125$ the behavior looks sensible but the RMSE is 0.650. With $\alpha_p = 0.25$ the RMSE is 0.300, the behavior induced is uniform.}
      \label{fig:exp3_derv_demon}
 \end{figure}

The {\em Error Derivative} has similar problems. The behavior of the Gradient Bandit with {\em Error Derivative} does not exhibit the expected shape under any configuration of the window length parameters we tested (see Figure \ref{fig:exp3_derv_window}). If we inspect the behavior changes as a function of the step-size parameter $\alpha_p$ (Figure \ref{fig:exp3_derv_demon}), then again we observe that slower learning is required for the {\em Error Derivative} reward to be highest for the drift actions, at the cost of high RMSE.

\subsection{Results with Introspective learners}
The results of our second experiment, this time with Autostep, are summarized in Figure \ref{fig:r4}. As before we plot the action selection probabilities to summarize the behavior. As before, {\em Weight Change}  and {\em Surprise} rewards now induce sensible action selection. As in Experiment One, {\em Expected Error} and {\em Error Derivative} do not induce the expected behavior because Autostep increases the effective learning rate of the LMS predictors. {\em UDE} induces similar behavior on this problem with or without Autostep. 

\begin{figure}
      \centering
      \includegraphics[width=0.92\linewidth]{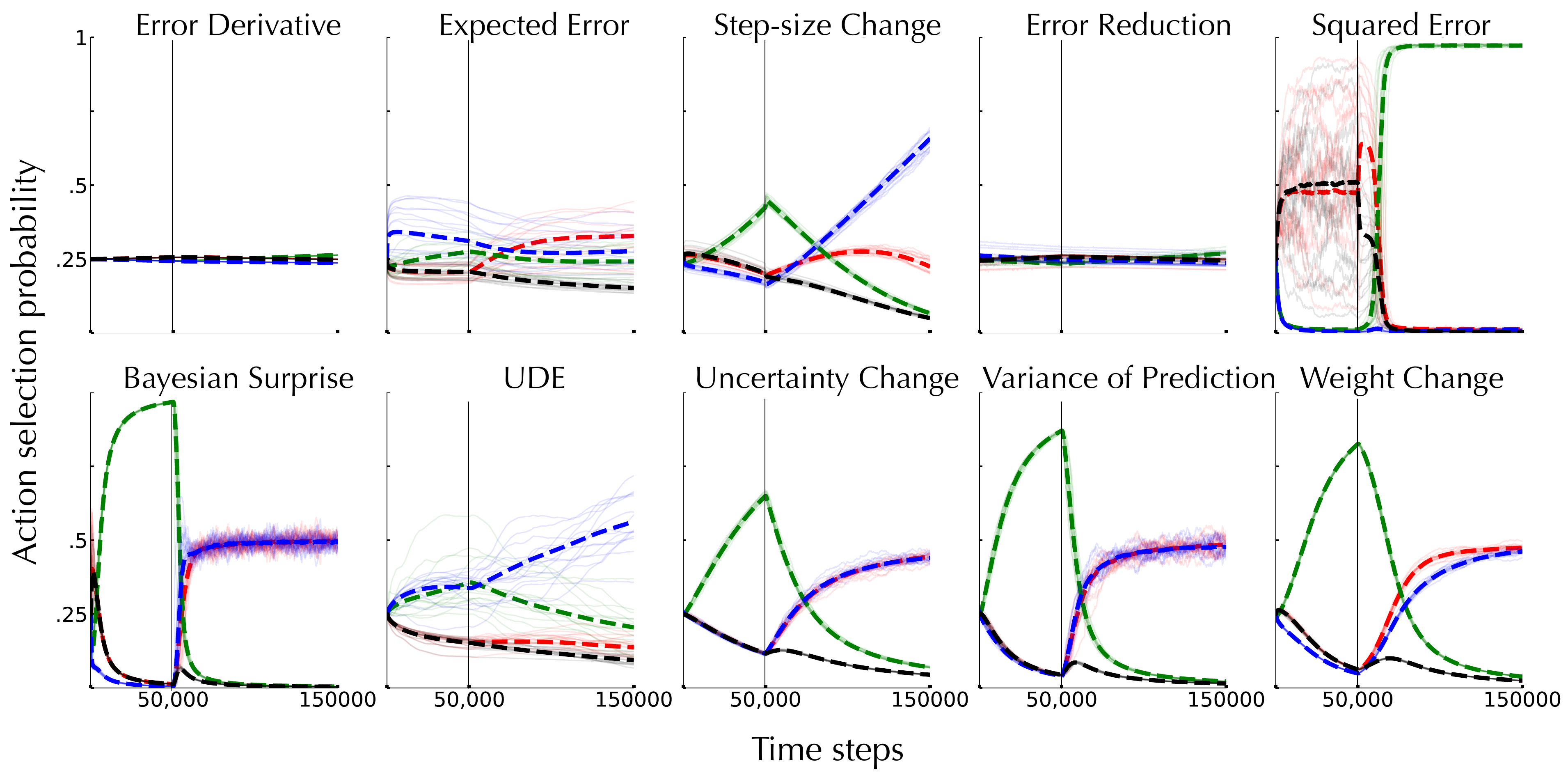}
     \caption{{\textbf{Behavior} in the \textbf{Switched Drifter-Distractor} problem with {\bf Introspective Learners},} 
      with the Gradient Bandit. 
      }
      \label{fig:r4}
 \end{figure}

\section{Experiment Three: Jumpy Eight-Action Problem}  

Our final experiment, in the Jumpy Eight-Action problem, quantitatively compares the best performing intrinsic rewards in a setting where the behavior agent should prefer several different actions. To achieve good performance, the agent must continually sample three actions, with different probabilities, and eventually ignore three noisy targets and two constant targets. The jumpy target follows a pattern increasing towards 50 and then decreasing toward -50 repeatedly. The idea was to design a problem where the intrinsic rewards based on variance would induce the wrong preferences over arms. The jumping-target exhibits high variance, as do the other two drifting targets. Intrinsic rewards based exclusively on the variance of the predictions might over-reward the jumpy-target action. 

This final experiment provides the quantitative comparison for performance, so we include a baseline uniform random behavior. There are eight targets and three of them continually drift. The drift is fast enough that wasting action selection will result in high error.  By the design of the experiment, therefore, uniform behavior should not be optimal, and we should expect some of the intrinsic rewards to significantly outperform uniform random behavior. We exclude rewards that did not succeed in inducing useful behavior in Experiments One and Two, and so have little hope to provide improvements here. This includes {\em Error Reduction} and {\em Squared Error} rewards. 

   \begin{figure}[h!]
      \centering
       \begin{tikzpicture}
  \node (img)  {\includegraphics[width=0.95\linewidth]{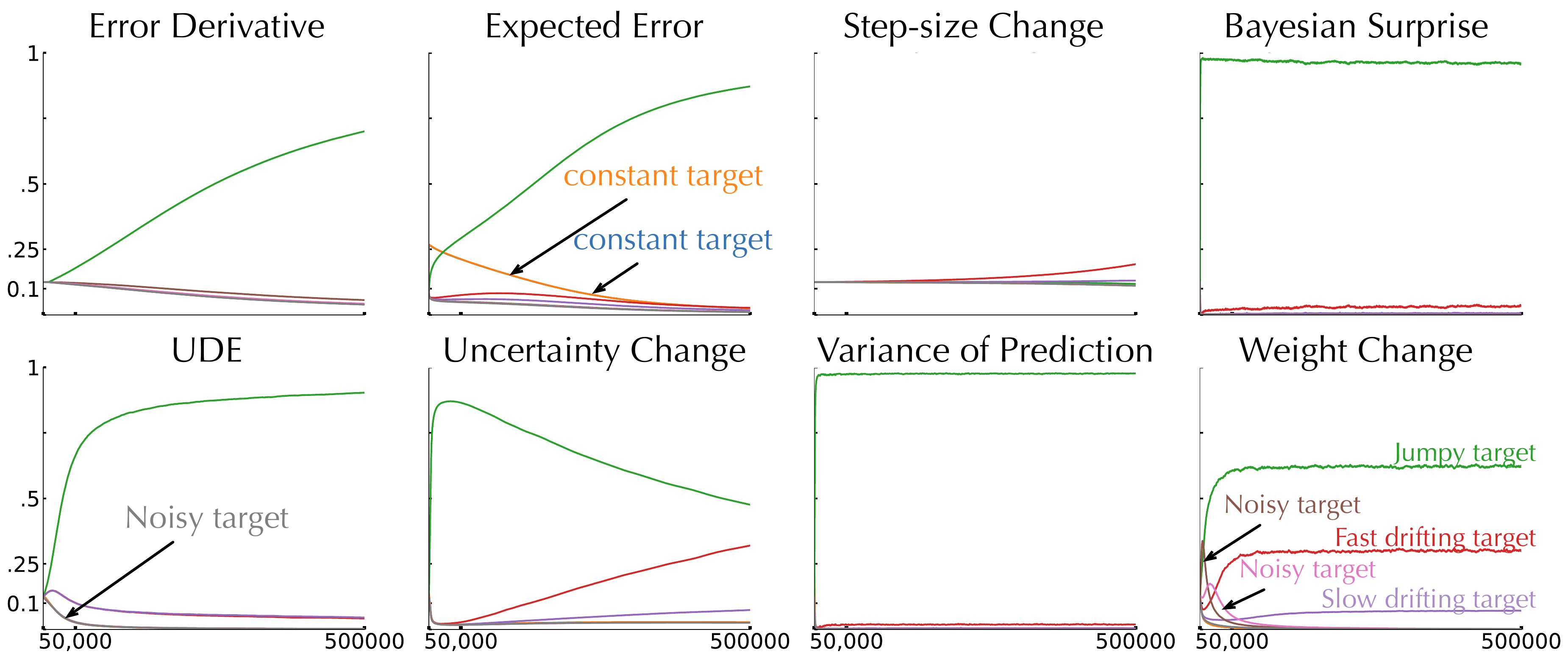}};
  \node[below=of img, node distance=0cm, yshift=1cm,font=\color{black}] {Time steps};
  \node[left=of img, node distance=0cm, rotate=90, anchor=center,yshift=-0.7cm,font=\color{black}] {Action selection probability};
 \end{tikzpicture}  
      \caption{\textbf{Behavior} in the \textbf{Jumpy Eight-Action} problem with \textbf{Introspective Learners}. Plotted is the probability of selecting each action versus time (averaged over 400 runs), with Autostep prediction learners. }
      \label{fig:exp5_pi}
\end{figure}
\begin{figure}[h!]
      \centering
      \includegraphics[width=0.7\linewidth]{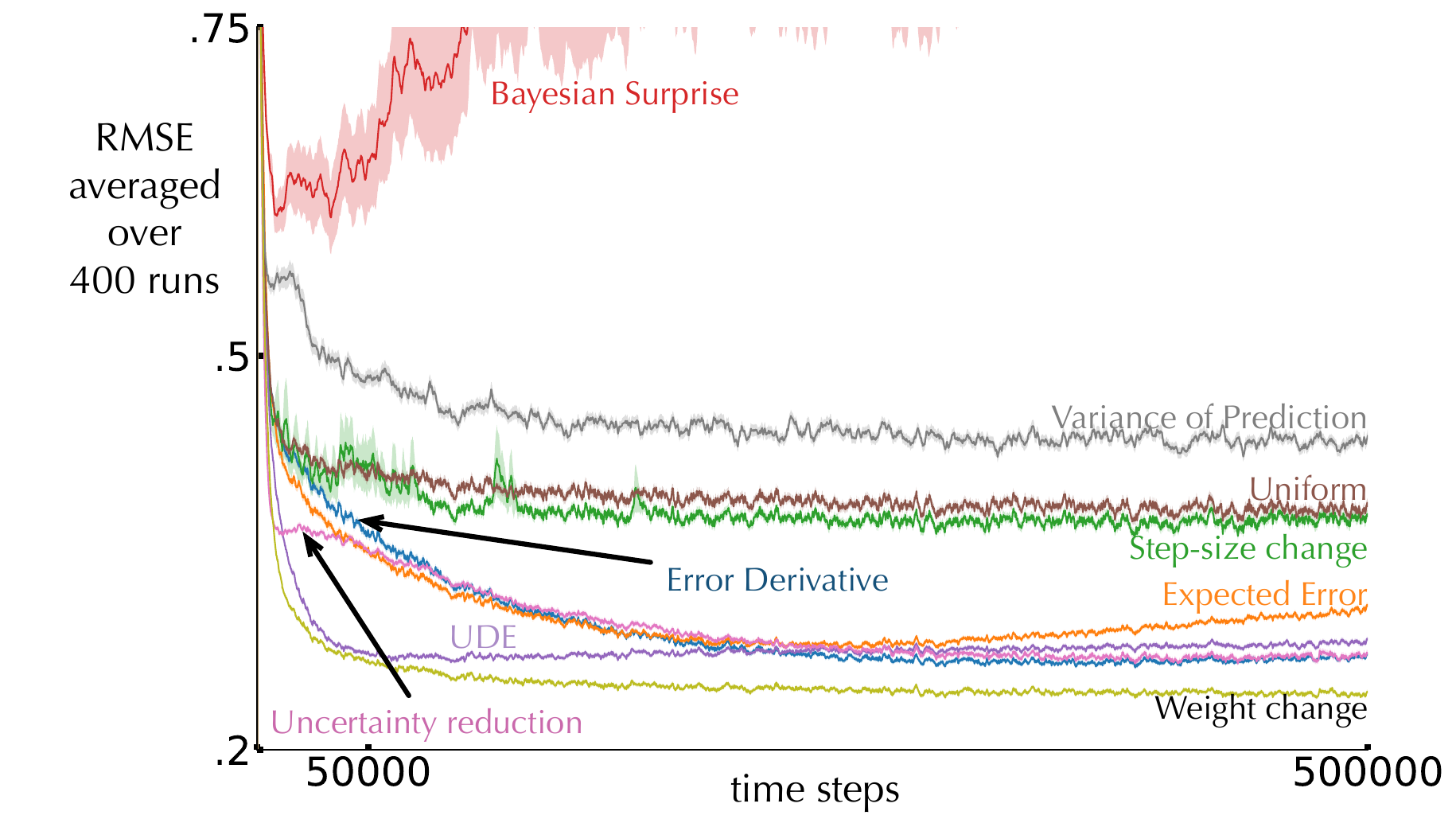}
      \caption{\textbf{RMSE} for the \textbf{Jumpy Eight-Action}  problem with \textbf{Introspective Learners}. Standard error of 400 runs is shown as shading around each line. We include the performance of the uniform behavior to provide a baseline of comparison. The behavior agent based on {\em Weight Change} learned the fastest and achieved the lowest RMSE at the end of learning. Most of the differences in the final RMSE are significantly different, in that the 95\% confidence intervals do not overlap (2 times the standard error), except for the difference between {\em Uniform} and {\em Step-size Change} and between {\em Uncertainty Reduction} and {\em Error Derivative}.}
      \label{fig:exp5_lc}
\end{figure}

Figure \ref{fig:exp5_pi} shows the behavior of the Gradient Bandit with eight different intrinsic rewards. We report the results using the hyper-parameter settings that resulted in the lowest total RMSE. We expect the behavior should initially select all the actions, eventually learning to ignore the constant and high-variance targets once they have been learned. Let us inspect the behavior with {\em Weight Change}. Initially, the actions corresponding to the high-error constant targets are preferred. Next the action corresponding to the high-variance target (with variance 1.0) is strongly preferred for a while, as is the other other high-variance target (with variance 0.5). The remaining noisy target (with variance 0.1), is easy to quickly learn and is not preferred after the beginning. Eventually the steady-state behavior selects the jumpy action most often, followed by the two drift actions. As you can see in Figure \ref{fig:exp5_lc}, the behavior induced by {\em Weight Change} significantly outperforms uniform action selection, and all other intrinsic rewards we tested.

Several rewards cause the behavior to over select the jumpy action, at the cost of under selecting the two drift actions. {\em Bayesian Surprise} and {\em Variance of Prediction} eventually select the jump action nearly 100\% of the time. In Figure \ref{fig:exp5_lc} we can see that both these rewards results in high RMSE compared to uniform action selection. This makes sense for intrinsic rewards based on variance because the variance of the jumpy target is significantly higher than the others. 
{\em Bayesian Surprise} likely did not work because our distributional assumptions are wrong---recall we assumed all targets were Gaussian with a moving estimate of the variance of the target, to provide an approximate Bayesian surprise for non-Bayesian learners. This is simply a problem with using Bayesian surprise outside its intended use-case---it should be used with Bayesian learners. But, as motivated above, we wanted to verify if a more explicit form of Bayesian surprise, rather than the much simpler {\em Weight Change}, could provide benefit, even for non-Bayesian learners. 

Many of the other rewards also over selected the jumpy action, though less excessively. \emph{Expected Error} takes a long time to induce a preference for the jumpy action, initially favoring the high-variance ones. \emph{UDE} induces a preference for the correct actions, but over-rewards the jumpy action---especially toward the end of the experiment---resulting in increasing RSME over time as we see in the upward trend of the error of \emph{UDE} in Figure \ref{fig:exp5_lc}. {\em Uncertainty Change} eventually induces reasonable action selection, though it starts by over selecting the jumpy action.  {\em Error Derivative}, after 500,000 steps, appears to be trending towards over selection of the jumpy action as well, while {\em Step-size Change} seems to be getting it all wrong. \emph{Weight Change} induces a preference for all three non-stationary targets after initially selecting all the actions. We cannot say if this is the optimal behavior, but it does cause the prediction learners to learn faster and achieve lower error compared with all other intrinsic rewards we tested.

\section{Adapting the Behavior of a Horde of Demons}
\label{horde}
The ideas and algorithms promoted in this paper may be more impactful when combined with policy-contingent, temporally-extended prediction learning. Imagine learning hundreds or thousands of off-policy predictions from a single stream of experience, as in the Unreal \citep{jaderberg2016reinforcement} and Horde \citep{sutton2011horde} architectures. In these settings, the behavior must balance overall speed of learning with prediction accuracy. That is, balancing action choices that generate off-policy updates across many predictions, with the need to occasionally choose actions in almost total agreement with one particular policy. In general we cannot assume that each prediction target is independent as we have done in this paper; selecting a particular sequence of actions might generate useful off-policy updates to several predictions in parallel \citep{white2012scaling}. There have been several promising investigations of how intrinsic rewards might benefit single (albeit complex) task learning \citep[see][]{pathak2017curiosity,hester2017intrinsically,tang2017exploration,colas2018gep,pathak2019self}. However, to the best of our knowledge, no existing work has studied adapting the behavior based on intrinsic rewards of a model-based or otherwise parallel off-policy learning system.

It seems clear that simple intrinsic reward schemes and the concept of an introspective learning system should scale nicely to these more ambitious problem settings. We could swap our stateless LMS learners for Q-learning with experience replay, or gradient temporal difference learning \citep{maei2010toward}. The Weight Change reward could be computed for each predictor with computation linear in the number of weights. It would be natural to learn the behavior policy with an average-reward actor-critic architecture, instead of the gradient bandit algorithm used here.  Finally, the notion of an introspective learner still simply requires that each prediction learner can adapt its rate of learning. This can be achieved with quasi second order methods like Adam \citep{kingma2015adam}, or extensions of the Autostep algorithm to the case of temporal difference learning and function approximation \citep{kearney2018tidbd,kearney2019learning,guenther2020examining,jacobsen2019meta}. It is not possible to know if the ideas advocated in this paper will work well in a large-scale off-policy prediction learning architecture like Horde, however they will certainly scale up. 

In this paper we focused on (non-Bayesian) mean prediction learners. 
This is a natural first step, as most prediction algorithms in reinforcement learning form point estimates (e.g., value functions). It is possible, however, to go beyond mean predictions, both for our testbed and for the Horde setting. For example, it is straightforward to use distributional prediction learners---those that estimate the distribution over the targets. For mean prediction learners we used the normed difference between weights before and after an update; for distributional learners, we could use a KL divergence between the distribution over targets before and after an update. Note that this differs from Bayesian surprise, which is the KL divergence between the distributions over the \emph{weights} before and after an update. We expect similar conclusions with distributional prediction learners, as conceptually the definition of \emph{Weight Change} simply uses a modified difference measure---a norm versus a divergence. The extension to Bayesian prediction learners---for either the mean or distribution over targets---is more complex, though improvements in Bayesian neural networks make it a more feasible line of investigation.

Maximizing intrinsic reward as presented in this paper is not a form of exploration, it's the objective of the learning system---the agent must explore in order to maximize the intrinsic reward. The intrinsic rewards do not provide a bonus to help improve exploration.
In our stateless prediction task, sufficient exploration was provided by the stochastic behavior policy.
This will not always be the case, and additional exploration will likely be needed. 
Efficient exploration is an open problem in reinforcement learning. Combining the ideas advocated in this paper with exploration bonuses or planning could work well, but this topic is left to future work.

\section{Conclusion}
The goal of this work was to systematically investigate intrinsic rewards for a multi-prediction setting. This paper has three main contributions. The first is a new benchmark suite for comparing intrinsic rewards. Our bandit-like task requires the learning system to demonstrate several important capabilities: avoiding dawdling on noisy outcomes, tracking non-stationary outcomes, and seeking actions for which consistent learning progress is possible. Second, we provide a survey of intrinsically motivated learning systems, and empirically investigated 10 different analogs of well-known intrinsic reward schemes. Finally, we found that simple intrinsic rewards based on amount of learning, can induce effective behavior---avoiding classic degenerative behavior---if the base prediction learners are introspective. Introspective prediction learners can decide for themselves when learning is done. Previous work focused on designing more complex intrinsic rewards to mitigate bad behavior. Our results suggest the opposite: we should focus on designing better learners and use simple (ideally parameter-free) intrinsic rewards. We found that intrinsic rewards based on amount of learning---like Weight Change---can perform well in problems specifically designed to distract the learning system. This work provides several new insights into the strengths and weakness of different intrinsic reward mechanisms, and may provide guidance for constructing larger more complex intrinsically motivated reinforcement learning systems where an extensive and systematic study like ours is not feasible. 

\acks{We would like to thank our generous funders for supporting this work, specifically the CIFAR Canada AI Chairs program and NSERC Discovery grant program. We would like to thank Rich Sutton for his ideas and insights that shaped this project early on. Finally, we would like to thank our colleagues at the Reinforcement Learning and Artificial Intelligence Lab at the University of Alberta and DeepMind Alberta for providing an exciting and stimulating environment for research. 
}

\vskip 0.2in
\bibliographystyle{apacite}
\bibliography{curiosity}

\end{document}